\pdfoutput=1

\documentclass[11pt]{article}

\PassOptionsToPackage{dvipsnames}{xcolor}
\usepackage[final]{assets/acl}

\usepackage[T1]{fontenc}
\usepackage[utf8]{inputenc}
\usepackage[nopatch=footnote]{microtype}  
\usepackage{times}                        
\usepackage{inconsolata}                  
\usepackage[scaled=0.85]{helvet}          
\usepackage{graphicx}                     
\makeatletter                             
    \DeclareRobustCommand*{\escapeus}[1]{%
    \begingroup\@activeus\scantokens{#1 }\endgroup}
    \begingroup\lccode`\~=`\_\relax
   \lowercase{\endgroup\def\@activeus{\catcode`\_=\active \let~\_}}
\makeatother
\newcommand{\myemph}[1]{\textsf{{\escapeus{#1}}}}

\usepackage{mathtools} 
\usepackage{amssymb}   
\usepackage{amsthm}    
\usepackage{bbm}       
\usepackage{bm}        
\usepackage{nicefrac}  

\usepackage{amsthm}       
\usepackage{thm-restate}  
\newtheorem{definition}{Definition}

\newtheorem{assumption}{Assumption}
\newtheorem{estimator}{Estimator}

\usepackage{enumitem}        
\usepackage{fnpct}           
\usepackage{tcolorbox}       
\usepackage[misc]{ifsym}     
\usepackage{url}             
\usepackage{booktabs}        
\usepackage{makecell}        
\usepackage{multirow}        
\usepackage{csquotes}        
\usepackage{xspace}          
\makeatletter                
\xspaceaddexceptions{\csq@qclose@i}
\makeatletter
\usepackage{wrapfig}

\usepackage[table-column-type=s]{siunitx}
\DeclareSIUnit\thousand{k}
\DeclareSIUnit\million{M}
\DeclareSIUnit\billion{B}
\DeclareSIUnit\trillion{T}
\DeclareSIUnit\x{x}
\DeclareSIUnit\percent{\%}
\DeclareSIUnit\hour{h}
\DeclareSIUnit\min{m}
\DeclareSIUnit\sec{s}
\DeclareSIUnit\gb{GB}
\newcommand{\integer}[1]{\num[mode = math, round-mode=places, round-precision=0, group-separator={,}, group-minimum-digits=4]{#1}\xspace}
\newcommand{\float}[2][1]{\num[group-digits=false, round-precision=#1, round-mode=places]{#2}\xspace}
\newcommand{\snum}[1]{\num{#1}\xspace}
\newcommand{\q}[2]{\qty[mode=math]{#1}{#2}\xspace}

\usepackage{tabularray}
\usepackage{cellspace}
\usepackage{setspace}
\newcommand\cincludegraphics[2][]{\raisebox{-0.28\height}{\includegraphics[#1]{#2}}}

\usepackage{cleveref}
\crefname{section}{\S}{\S\S}
\Crefname{section}{\S}{\S\S}
\crefformat{section}{\S#2#1#3}
\crefname{appendix}{App.}{}
\crefname{figure}{Fig.}{Fig.}
\crefname{table}{Table}{Tables}
\crefname{equation}{eq.}{eqs.}
\crefname{theorem}{Thm.}{Thm.}

\crefname{proposition}{Proposition}{Propositions}
\crefname{assumption}{Assump.}{Assumps.}
\crefname{definition}{Definition}{Definitions}
\crefname{estimator}{Estimator}{Estimators}

\usepackage{scalerel}    
\usepackage{subcaption}  
\usepackage{soul}        
\makeatletter
 \def\SOUL@hlpreamble{%
 \setul{0pt}{2.6ex}%
 \let\SOUL@stcolor\SOUL@hlcolor
 \SOUL@stpreamble
 }
\makeatother

\DeclareUnicodeCharacter{0307}{Ş}
\definecolor{Gray}{gray}{0.9}
\definecolor{black}{gray}{0}

\usepackage[disable]{todonotes}

\newcommand{\defn}[1]{\textbf{#1}}
\newcommand{\vs}{\emph{vs.}\@\xspace}

\DeclareMathOperator*{\argmin}{argmin}
\DeclareMathOperator*{\argmax}{argmax}
\DeclareMathOperator*{\expect}{\mathbb{E}}

\DeclareMathOperator*{\generalaggfunction}{\mathbb{M}}

\newcommand{\defeq}{\mathrel{\stackrel{\textnormal{\tiny def}}{=}}}
\newcommand{\smallcdots}{\cdots}
\newcommand{\smalldots}{...}
\newcommand{\size}[1]{{\left\vert #1 \right\vert}}
\newcommand{\one}{\mathbbm{1}}

\definecolor{blind_blue}{HTML}{547FEF}
\definecolor{blind_magenta}{HTML}{DC267F}
\definecolor{blind_yellow}{HTML}{FFB000}
\definecolor{blind_orange}{HTML}{FE6100}
\newcommand{\subwordcolor}{purple}
\newcommand{\charcolor}{blind_blue}
\newcommand{\mergecolor}{blind_orange}

\newcommand{\textcharcolor}[1]{\textcolor{\charcolor}{#1}}
\newcommand{\textsubwordcolor}[1]{\textcolor{\subwordcolor}{#1}}
\newcommand{\textmergecolor}[1]{\textcolor{\mergecolor}{#1}}

\newcommand{\bpe}{\myemph{BPE}\xspace}
\newcommand{\wordpiece}{\myemph{WP}\xspace}
\newcommand{\bpewp}{\myemph{BPE2WP}\xspace}
\newcommand{\tok}{\tau}
\newcommand{\detok}{\ensuremath{\rotatebox[origin=c]{180}{$\tok$}}}
\newcommand{\toklongest}{\tok_{\myemph{long}}}
\newcommand{\tokbup}{\tok_{\uparrow}}
\newcommand{\tokeniser}{\mathbbm{T}}

\newcommand{\objectivefunc}{\phi}
\newcommand{\objectivebpe}{\objectivefunc_{\myemph{bpe}}}
\newcommand{\objectivewordpiece}{\objectivefunc_{\myemph{wp}}}
\newcommand{\countfunc}{\#}

\newcommand{\charstring}[2][true]{%
  {\color{\charcolor}\ensuremath{%
    \ifthenelse{\equal{#1}{true}}{\mathit{\text{``}\escapeus{#2}\text{''}}}{\text{``}#2\text{''}}}}%
  \xspace
}

\newcommand{\subwordspan}[2][true]{%
  {\color{\subwordcolor}\ensuremath{%
    \ifthenelse{\equal{#1}{true}}{\mathit{\text{``}\escapeus{#2}\text{''}}}{\text{``}#2\text{''}}}}%
  \xspace
}

\newcommand{\charspan}[2][true]{%
  {\color{\subwordcolor}\ensuremath{%
    \ifthenelse{\equal{#1}{true}}{\mathit{\text{``}\escapeus{#2}\text{''}}}{\text{``}#2\text{''}}}}%
  \xspace
}

\newcommand{\ch}[1][]{{\color{\charcolor}c}\ifx#1\relax \else _{{#1}}\fi}

\newcommand{\chs}[1][]{\mathbf{\color{\charcolor}c}\ifx#1\relax \else _{{#1}}\fi}

\newcommand{\subwordstring}[2][true]{%
  {\color{\subwordcolor}\ensuremath{%
    \ifthenelse{\equal{#1}{true}}{\mathit{\langle \escapeus{#2} \rangle}}{\langle #2 \rangle}}}%
  \xspace
}
\newcommand{\subword}[1][]{%
  {\color{\subwordcolor} v}%
  \ifx#1\relax \else _{{\color{\subwordcolor}{#1}}}\fi
}
\newcommand{\subwords}[1][]{%
  \mathbf{\subword}%
  \ifx#1\relax \else _{{\color{\subwordcolor}{#1}}}\fi
}
\newcommand{\vocab}[1][]{%
  {\color{\subwordcolor} \mathcal{V}}%
  \ifx#1\relax \else _{{\color{\subwordcolor}{#1}}}\fi
}
\newcommand{\vocabsize}{K}

\newcommand{\mergestring}[3][false]{%
  {\color{\mergecolor}\ensuremath{%
    \ifthenelse{\equal{#1}{true}}{\mathit{\escapeus{#2}}\mathop{\circledcirc}\mathit{\escapeus{#3}}}
    {#2\mathop{\circledcirc}#3}}}%
  \xspace
}
\newcommand{\merge}[1][]{%
  {\color{\mergecolor} m}%
  \ifx#1\relax \else _{{\color{\mergecolor}{#1}}}\fi
}
\newcommand{\merges}[1][]{%
  \mathbf{\merge}%
  \ifx#1\relax \else _{{\color{\mergecolor}{#1}}}\fi
}
\newcommand{\allmerges}[1][]{%
  \mathbf{{\color{\mergecolor}\widetilde{\merge}}}%
  \ifx#1\relax \else _{{\color{\mergecolor}{#1}}}\fi
}
\newcommand{\mergefunc}[1]{\mathrm{merge}_{#1}}

\newcommand{\mergepair}[1][]{\subword^{\textcolor{\subwordcolor}{\textnormal{\scalebox{0.75}{[#1]}}}}}

\newcommand{\dataset}{\mathcal{D}}

\newcommand{\alphabet}{{\color{\charcolor}\Sigma}}

\newcommand{\stringset}{{\color{\charcolor}\alphabet^*}}
\newcommand{\stringsetplus}{\color{\charcolor}\alphabet^+}
\newcommand{\subwordsset}{{\color{\subwordcolor}\vocab^*}}

\newcommand{\eos}{\myemph{eos}\xspace}

\newcommand{\vtheta}{\boldsymbol{\theta}}
\newcommand{\ptheta}{p_{\scaleto{\vtheta}{5.5pt}}}
\newcommand{\ptok}{p^{\scaleto{\tokeniser}{5pt}}}
\newcommand{\pthetatok}{\ptheta^{\scaleto{\tokeniser}{5pt}}}

\newcommand{\pthetaprime}{p_{\scaleto{\vtheta'}{5.5pt}}}

\newcommand{\pthetaprimetokprime}{\pthetaprime^{\scaleto{\tokeniser'}{5pt}}}

\newcommand{\chssubword}{\chs[\subword]}

\newcommand{\treatment}{W}

\newcommand{\expectpotentialoutcome}[1][]{Y_{#1}(\subword)}

\newcommand{\otherpotentialoutcome}[1][]{Y^{\generalaggfunction}_{#1}(\subword)}
\newcommand{\singlecontextpotentialoutcome}[1][]{Y^{\chs[<t]}_{#1}(\subword)}
\newcommand{\observedoutcome}{Y_{\!\scaleto{\textnormal{\myemph{obs}}}{4.5pt}}(\subword)}
\newcommand{\causaleffect}{\psi}

\newcommand{\atemerge}{\causaleffect_{\subword}}
\newcommand{\ate}{\causaleffect}
\newcommand{\rdestimand}{\ate_{\textnormal{\myemph{RD}}}}
\newcommand*{\circled}[1]{\tikz[baseline=(char.base)]{\node[shape=circle,draw,inner sep=1pt] (char) {\normalfont{\small #1}};}}
\newcommand{\orderfunc}{\gamma}
\newcommand{\runningvar}{\orderfunc_{\subword}}
\newcommand{\ordermerge}{\orderfunc_{\subword}}
\newcommand{\cutoff}{\vocabsize}
\newcommand{\cutoffdelta}{k}
\newcommand{\funcrel}{f}
\newcommand{\rdestimator}{\widehat{\ate}_{\textnormal{\myemph{RD}}}}
\newcommand{\aterunningvar}{\ate_{\runningvar}}
\newcommand{\late}{\ate_{\runningvar}}
\newcommand{\observedset}{\mathcal{K}}
\newcommand{\residual}{\eta}

\newcommand{\minipile}{\myemph{MiniPile}\xspace}
\newcommand{\fineweb}{\myemph{Fineweb-Edu}\xspace}

\title{Causal Estimation of Tokenisation Bias}

\newcommand{\camid}{{\includegraphics[scale=0.018]{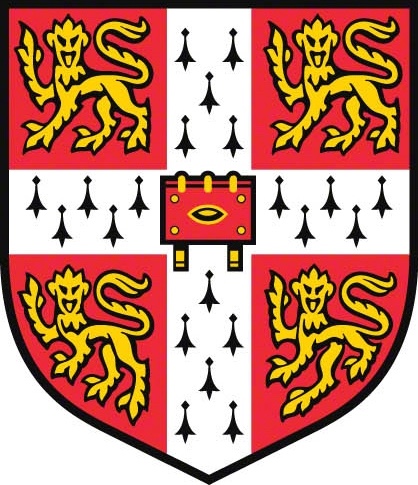}}}
\newcommand{\ethid}{{\includegraphics[scale=0.028]{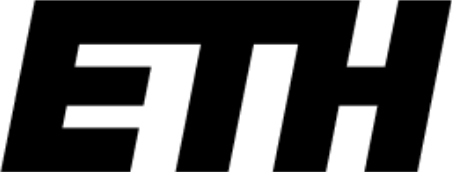}}}
\newcommand{\camemailadress}[1]{\href{mailto:#1@cam.ac.uk}{#1}}

\author{
    Pietro Lesci\,$^\camid_{\text{\textnormal{\Letter}}}$ 
    Clara Meister\textsuperscript{\ethid} 
    Thomas Hofmann\textsuperscript{\ethid} 
    Andreas Vlachos\,$^\camid_{\text{\textnormal{\Letter}}}$
    Tiago Pimentel\,$^\ethid_{\text{\textnormal{\Letter}}}$ \\
    \textsuperscript{\camid}University of Cambridge \quad \textsuperscript{\ethid}ETH Z\"urich \\
    $^{\text{\Letter}}${\myemph{\{\camemailadress{pl487}, \camemailadress{av308}\}\textcolor{darkblue}{@cam.ac.uk}}, \href{mailto:tiago.pimentel@inf.ethz.ch}{\myemph{tiago.pimentel@inf.ethz.ch}}} \\
    \\
    \begin{tblr}{colspec = {Q[c,m] Q[c,m]}, colsep=10pt, stretch=0}
        \cincludegraphics[width=1.1em, keepaspectratio]{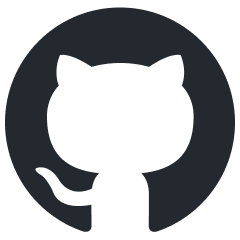} {\fontsize{11pt}{11.5pt}\selectfont\href{https://github.com/pietrolesci/tokenisation-bias}{\myemph{pietrolesci/tokenisation-bias}}}
        & \cincludegraphics[width=1.em, keepaspectratio]{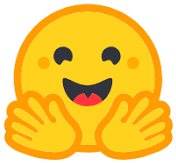} {\fontsize{11pt}{11.5pt}\selectfont\href{https://huggingface.co/collections/pietrolesci/tokenisation-bias-66d5d0b40cb82a2d789b19db}{\myemph{pietrolesci/Tokenisation-Bias}}}
    \end{tblr}
}

\begin{document}

\maketitle

\begin{abstract}
Modern language models are typically trained over subword sequences, but ultimately define probabilities over character-strings. Ideally, the choice of the tokeniser---which maps character-strings to subwords---should not affect the probability assigned to the underlying character-string; in practice, it does.
We define this mismatch as \defn{tokenisation bias}.
In this work, we quantify one particular type of tokenisation bias: the effect of including or not a subword (e.g., \subwordstring{hello}) in a tokeniser’s vocabulary on the probability a trained model assigns to the corresponding characters (i.e., \charstring{hello}).
Estimating this effect is challenging because each model is trained with only one tokeniser.
We address this by framing tokenisation bias as a causal effect and estimating it using the regression discontinuity design. 
Specifically, we exploit the fact that tokenisation algorithms rank subwords and add the first $\vocabsize$ to a tokeniser's vocabulary, where $\vocabsize$ is an arbitrary cutoff point.
As such, we can estimate a causal effect by comparing similar subwords around this cutoff.
Experimentally, we find that tokenisation consistently affects models' outputs across scales, vocabularies, and tokenisers. 
Notably, a subword's presence in a small model's vocabulary may increase its characters' probability by up to \integer{17} times, highlighting tokenisation as a key design choice in language modelling.
\end{abstract}

\section{Introduction}
\label{sec:intro}

Language models (LMs) define probability distributions over \textcharcolor{\defn{character-strings}} (e.g., \charstring{hello}), i.e., finite sequences of characters\footnote{We use characters as a generic term to refer to either raw bytes (e.g., in ASCII), Unicode symbols, or graphemes.} from an alphabet.
Directly modelling character-strings, however, can be inefficient, as it might require processing long sequences.
To improve computational efficiency, modern LMs (e.g., \citealp{touvron2023llama}) typically model \textsubwordcolor{\defn{subwords-strings}} instead (e.g., \subwordstring{he, llo}); these subword-strings are produced by a \defn{tokeniser}, where each subword represents a sequence of characters.
While in practice LMs thus output distributions over subword-strings, we can still map them back to distributions over character-strings \citep{pimentel-meister-2024-compute,phan-etal-2025-exact}.

\begin{figure}[!t]
    \centering
    \includegraphics[width=\columnwidth]{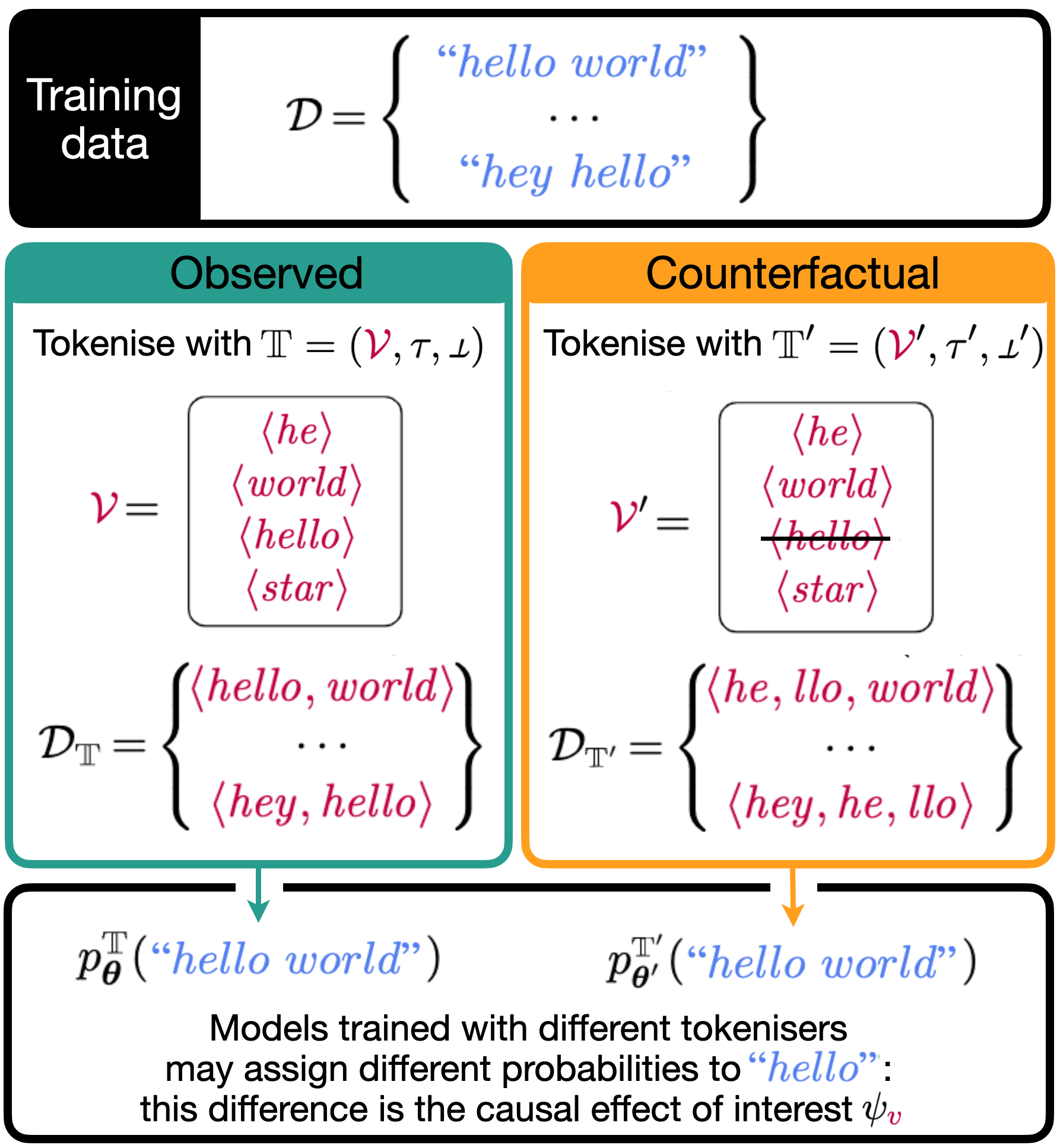}
    \vspace{-18pt}
    \caption{\textbf{Tokenisation bias.} Consider subword $\subwordstring{hello}$. If it is included in the vocabulary, then each occurrence of $\charstring{hello}$ in the training data is represented as a single subword; otherwise, it is split into two subwords, i.e., $\subwordstring{he, llo}$. If two models are trained under these two settings, the difference in the probability they assign to $\charstring{hello}$ is the tokenisation bias we aim to estimate.
    }
    \label{fig:schematic}
    \vspace{-10pt}
\end{figure}

Notably, a tokeniser's behaviour is constrained by its \defn{vocabulary}: a finite set of subwords which it may produce.
Given a character-string, a tokeniser deterministically maps it to a sequence of subwords from its vocabulary.
Different tokenisers (with different vocabularies) thus map
character-strings to different subword-strings (e.g., map \charstring{hello} to either \subwordstring{he, llo} or \subwordstring{hello}). 
While an ideal LM would assign the same probability to a string regardless of tokenisation (\cref{sec:bias_in_perfect}), in practice, tokenisation is an influential design choice for model performance \citep[\textit{inter alia}]{rust-etal-2021-good,Toraman_2023,ali-etal-2024-tokeniser}.
How tokenisation affects trained models, however, remains poorly understood.

In this work, we first define the effect of a tokenisation choice on model behaviour as \defn{tokenisation bias}.
We then quantify one particular type of tokenisation bias: the effect of including or not a subword (e.g., \subwordstring{hello}) in a tokeniser's vocabulary on the log-probability a trained model assigns to its characters (i.e., \charstring{hello}).
While tokenisation bias may be an intuitive concept, estimating it is not straightforward.
We cannot simply compare  a model's log-probabilities on in- \vs out-of-vocabulary words (e.g., \charstring{hello} \vs \charstring{appoggiatura}); as vocabularies are built based on, e.g., frequency, there are systematic differences between these groups.
Further, we cannot compare log-probabilities on the same word under different tokenisations (e.g., \subwordstring{he, llo} or \subwordstring{hello}); as the model only sees one of them during training.\footnote{A brute-force approach---e.g., training separate models under the different possible tokenisers---is (i) computationally impractical and (ii) estimates expected bias for a LM family, rather than the tokenisation bias in a specific model instance.}\looseness=-1

To quantify tokenisation bias, we first frame it as a \defn{causal effect}, asking: how would a word’s log-probability change if its subword was removed from the vocabulary (see \cref{fig:schematic})?
This question requires comparing two values: 
(i) an \defn{observed value}, the log-probability the model assigns the word in question;
(ii) a \defn{counterfactual value}, the log-probability the model \emph{would} assign the word had it not been in the vocabulary.
We estimate this causal effect by noting that tokenisers' vocabularies are typically built incrementally.
This effectively creates a ranking over subwords, which is arbitrarily cut off by a fixed vocabulary size: e.g., the first \q{32}{\thousand} subwords under this ranking are included in the vocabulary; later ones are excluded.
Regression discontinuity (RD) design \citep{thistlewaite1960regression} then gives us a principled way to estimate our causal effect: we can compare subwords near the cutoff---as they are bound to have similar features---using subwords before the cut-off to estimate our observed values, and the ones after the cutoff to input the counterfactual value.

Experimentally, we show that tokenisation consistently affects LM performance across vocabulary sizes, model scales, and tokenisers. Text represented as a single subword receives more log-probability (on average) than if split into two.
Further, this bias grows during training and in models under \q{100}{\million} parameters, it reaches \float[2]{2.88} nats for \bpe and \float[2]{2.51} for \wordpiece.
In \q{850}{\million}-parameter models, this bias is around \integer{1} nat: a subword's characters would thus be assigned roughly \float[1]{2.7} times less probability if it were not in the tokeniser's vocabulary.

\vspace{-3pt}
\section{Tokenisation}
\label{sec:tokenisation}
\vspace{-1pt}

Let $\ch \in \alphabet$ be a \defn{character}, where $\alphabet$ is an alphabet representing either a set of bytes, or the graphemes in a language (including whitespace and punctuation).
Further, let $\chs \in \stringset$ be a \textcharcolor{\defn{character-string}}, i.e., a finite sequences of characters; we will represent these strings as $\chs = \charstring[false]{\ch[1]\ch[2]\smallcdots\ch[\size{\chs}]}$.
The goal of language modelling is to learn a language's probability distribution over character-strings. 
We write these probabilities autoregressively as:
\begin{align}\label{eq:ch_dist}
    p(\chs) = p(\eos \mid \chs)\,\prod_{t=1}^{\size{\chs}} p(\ch[t] \mid \chs[<t])
\end{align}
where $\eos \notin \alphabet$ is a special symbol denoting the end of a string.\footnote{The \eos symbol is required to define probability distributions over strings autoregressively \citep{du-etal-2023-measure}.} 
Most modern LMs, however, do not model \cref{eq:ch_dist} directly.
Rather, they model distributions over \textsubwordcolor{\defn{subword-strings}} $\subwords\!\in\!\subwordsset$:%
\begin{align} \label{eq:subword_dist}
    p(\subwords) = p(\eos \mid \subwords)\,\prod_{t=1}^{\size{\subwords}} p(\subword_t \mid \subwords_{<t})
\end{align}
In words, subword-strings are finite sequences of subwords $\subword \in \vocab $, where $\vocab$ is a finite set typically called a vocabulary;
we thus represent these strings as  $\subwords = \subwordstring[false]{\subword_1, \smalldots, \subword_{\size{\subwords}}}$.
As we explain later, subwords represent character spans, denoted here as $\chs_{\subword} = \subword$.
How do \cref{eq:ch_dist,eq:subword_dist} connect?
This is the job of a tokeniser, which we describe next.

\subsection{Characters-strings \texorpdfstring{$\leftrightarrow$}{and} Subword-strings}
\vspace{-1pt}

\defn{Tokenisers} map character-strings to subword-strings and back, and can be formally defined as a tuple $\tokeniser \defeq (\vocab, \tok, \detok)$. 
The first item in this tuple is a \defn{vocabulary}: a finite set of character-spans $\vocab \subset \stringsetplus$, whose elements are called subwords $\subword \in \vocab$.\footnote{To ensure any character-string can be represented with $\vocab$, the original alphabet is typically included in this set: $\alphabet \subseteq \vocab$.} 
The second item is known as the \defn{tokenisation function}: a function $\tok\colon \stringset \to \subwordsset$ which maps character-strings to subword-strings. 
Finally, the third item in $\tokeniser$ is a \defn{detokenisation function}: a function $\detok\colon \subwordsset \to \stringset$ which maps subword-strings back to character-strings.

As subwords represent character-spans, detokenisation is typically defined as the simple concatenation of the subwords in a subword-string:
\begin{align}
    \detok(\subwords) \defeq \subword_1 \circ \subword_2 \circ \cdots \circ \subword_{\size{\subwords}}
\end{align}
e.g., $\detok(\subwordstring{he,llo}) = \charstring{hello}$. 
Further, to guarantee that a tokeniser is lossless, $\tok$ is designed such that detokenising its output returns the original character-string:  i.e., $\chs = \detok(\tok(\chs))$. 
We describe two popular tokenisation functions in \cref{app:tok_functions}.

\subsection{Language Modelling}

Language modelling is the task of defining a probability distribution over character-strings, arguably being the most prominent use case of tokenisation.
During LM training, we typically start with a dataset of character-strings: $\dataset = \{\chs_{n}\}_{n=1}^{N}$, where $\chs_{n} \sim p(\chs)$.
We then use a tokeniser $\tokeniser$ to convert each of these into a subword-string: $\subwords_n = \tok(\chs_n)$, which are the actual inputs to our model.
Given this setup, we can deduce a distribution over subword-strings from the tokeniser:
\begin{align}\label{eq:subword_ch_connection}
    \ptok\!(\subwords) = 
    \begin{cases}
        p(\chs) & \text{if} \;\; \subwords = \tok(\chs) \\
        0 & \text{otherwise}
    \end{cases}
\end{align}
where we denote this distribution as $\ptok$ now to make its dependence on $\tokeniser$ explicit.
\Cref{eq:subword_ch_connection} thus connects \cref{eq:ch_dist,eq:subword_dist}.
Notably, training an LM on strings $\subwords_n$ produces a model $\pthetatok(\subwords)$---which ideally should be a good approximation of $\ptok\!(\subwords)$.
We can then recover a character-level distribution $\pthetatok(\chs)$ from this model by marginalising over all $\subwords$ which are detokenised to this character-string:
\begin{align}
    \pthetatok(\chs) = 
        \sum_{\subwords \in \subwordsset} \pthetatok(\subwords) \; \one\{\chs = \detok(\subwords)\}
\end{align}
Notably, while this equation depends on an infinite sum over $\subwordsset$, it can be approximated efficiently \citep{phan-etal-2024-understanding,phan-etal-2025-exact,vieira-etal-2024-language}.

\section{Tokenisation Bias}\label{sec:tok_bias}

Now, consider tokeniser $\tokeniser = (\vocab, \tok, \detok)$ and 
assume we use it for training a LM; out data-generating distribution is thus $\ptok$ (from \cref{eq:subword_ch_connection}) and our LM is $\pthetatok$.
Further, consider a specific subword $\subword$ in this tokeniser's vocabulary.
We can measure how well the model predicts this subword's associated characters $\chs_{\subword}$ in a given context $\chs_{<t}$ as how large of a probability it assigns to this sequence of characters $\pthetatok(\chs_{\subword} \mid \chs_{<t})$.
We are interested in studying how this probability would change had $\subword$ not been in $\vocab$.
More generally, we can ask: how does our choice of tokeniser $\tokeniser$ bias our trained model $\pthetatok$?
We term this tokenisation bias.\footnote{We note that \citet{phan-etal-2024-understanding} use the term tokenisation bias to describe an issue arising when one applies $\tok$ to a prefix of $\chs$: namely, when $\tok(\chs_{<t})$ is not necessarily a prefix of $\tok(\chs)$. We interpret this not as a bias, but as a misuse of tokenisers by the LM's users. We will thus term \citeauthor{phan-etal-2024-understanding}'s analysed phenomenon tokenisation mismatch here, instead of bias.}

\begin{definition}\label{def:tokenisation_bias}
    The \defn{tokenisation bias} induced by a \emph{property of the tokeniser} is its effect on a model's ability to predict some character-string $\chs$.
\end{definition}

We can define different \emph{kinds} of tokenisation bias, considering different properties of a tokeniser.
The property we will focus on here is the inclusion of a subword in its vocabulary, i.e., $\subword \!\in\! \vocab$, and how it affects predictions on its characters, $\chs_{\subword}$.

\begin{definition}\label{def:tokenisation_bias_example}
    The \defn{tokenisation bias of \boldsymbol{$\subword\!\in\!\tokeniser$}} is its effect on a model's ability to predict $\chs_{\subword}$.\footnote{We use the shorthand notation $\subword \in \tokeniser$ to denote that $\subword$ is in the tokeniser’s vocabulary $\vocab$.}
\end{definition}

We now frame this tokenisation bias as a \defn{causal effect}: it compares the performance of an \emph{observed} model $\pthetatok$, with the performance of a \emph{counterfactual} model $\pthetaprimetokprime$.
This counterfactual model represents what our model \emph{would} have been if trained with tokeniser \smash{$\tokeniser'= (\vocab', \tok', \detok')$} which did not include subword $\subword$.
We define this causal effect formally in the next section.
First, however, note that: 
\begin{align}
    \underbrace{\subword = \tok(\chssubword),}_{\text{observed subword}} \qquad \underbrace{\subwords^{\prime\chssubword} = \tok'(\chssubword)}_{\text{counterfactual subwords}}
\end{align}
where \smash{$\subwords^{\prime\chssubword}$} is the sequence of subwords which would be used to represent \smash{$\chssubword$} in the counterfactual case.
Assuming no tokenisation mismatch \citep{phan-etal-2024-understanding}, i.e., that  $\tok(\chs_{<t})$ is a prefix of $\tok(\chs)$ in this context (and similarly under \smash{$\tok'$}), we then have:
\begin{align}
    \pthetatok(\chssubword \mid \chs_{<t}) &= \pthetatok(\subword \mid \tok(\chs_{<t})) \\
    \pthetaprimetokprime(\chssubword \mid \chs_{<t}) &= \prod_{t'=1}^{\size{\subwords^{\prime\chssubword}}} \pthetaprimetokprime(\subwords^{\prime\chssubword}_{t'} \mid \tok'(\chs_{<t}) \circ \subwords^{\prime\chssubword}_{<t'}) \nonumber
\end{align}
In words, we can compute $\pthetaprimetokprime(\chssubword \mid \chs_{<t})$ by multiplying the probabilities of subwords in $\subwords^{\prime\chssubword}$.

\subsection{A Subword's Causal Effect}

We now define tokenisation bias in terms of potential outcomes \citep{rubin-1974-estimating, rubin-2005-causal}; this framework allows us to formally describe the causal effect of a \defn{treatment} on some target \defn{outcome}.
In our case, we want to estimate the causal effect of including a subword $\subword$ in tokeniser $\tokeniser$ on a model $\pthetatok$'s ability to predict the character-string $\chs_{\subword}$.
Before defining the effect of interest, however, we must formally introduce its treatment and outcome.
We define a \defn{treatment assignment variable} as:
\begin{align}\label{eq:treatment}
    \treatment_{\subword} \defeq \one\{\subword \in \tokeniser\}
\end{align}
and we define a \defn{potential outcome variable} as:
\begin{align}\label{eq:potential_outcome}
    \expectpotentialoutcome[\treatment] \defeq \expect_{\chs[<t]}\left[
    \log \pthetatok(\chssubword \mid \chs[<t])
    \right]
\end{align}
where we quantify a model's ability to predict a character-string $\chssubword$ as its log-probability averaged across contexts $\chs[<t]$.\footnote{
Our framework also generalises beyond mean effects across contexts. For instance, replacing the expectation in \cref{eq:potential_outcome} with a standard deviation yields the tokenisation bias of  $\subword\!\in\!\tokeniser$ on a model's variability when predicting $\chs_{\subword}$.
Similarly, defining $\singlecontextpotentialoutcome[\treatment] = \log \pthetatok(\chssubword \mid \chs[<t])$ allows us to study a context-specific tokenisation bias instead of aggregated ones.
}
Given these definitions, we can write the causal effect of interest as follows.

\begin{definition}\label{def:ite_estimand}
    The \defn{causal effect of a subword} $\subword$ on a model's ability to predict character-string $\chssubword$ is:
    \begin{align}\label{eq:ite_estimand}
        \atemerge \defeq \underbrace{\expectpotentialoutcome[1]}_{\substack{\text{performance on $\chssubword$} \\ \text{when trained with $\subword$}}} - \underbrace{\expectpotentialoutcome[0]}_{{\substack{\text{performance on $\chssubword$} \\ \text{when trained without $\subword$}}}}\hspace*{-1.5em}
    \end{align}
\end{definition}

Importantly, a model is trained with a tokeniser that either includes $\subword$ or does not.
Thus, for any model, we can only observe either $\expectpotentialoutcome[1]$ or $\expectpotentialoutcome[0]$ with the other term being counterfactual.
The causal analysis literature provides us with methods---and the relative assumptions that need to be met---to impute the counterfactual term from observable outcomes in a principled way.
\cref{def:ite_estimand} represents an individual treatment effect (ITE), as it defines the causal effect of a specific subword.
Methods to estimate an ITE, however, require strong assumptions, if at all applicable \citep{lu-etal-2018-estimating}.
To avoid this requirement, we thus focus on average effects, as is common in the econometrics literature \citep{angrist-pischke-2015-mastering}.

\begin{definition}\label{def:ate_estimand}
The \defn{expected effect of a subword} on the model's ability to predict its character-string is:\looseness=-1
    \begin{align}\label{eq:ate_estimand}
        \ate \defeq \expect_{\subword}\left[
        \expectpotentialoutcome[1] - \expectpotentialoutcome[0]
        \right]
    \end{align}
\end{definition}

In other words, we simplify the causal estimation problem by focusing on the expected causal effect across a population of subwords.
The causal effect for each specific subword $\subword$ might thus be larger or smaller than this expected effect; the expected effect we measure, however, gives us a best guess (in terms of mean squared error) of what their effect would be, given no additional information.

\subsection{Tokenisation Bias in Perfect and in Untrained Models}
\label{sec:bias_in_perfect}

We now analyse this causal effect under \defn{perfect models}: LMs for which $\pthetatok$ exactly matches the data-generating distribution $\ptok$.

\begin{restatable}{theorem}{noeffecttheorem}  \label{thm:no_effect}
    Assume we have a training process that---regardless of the choice of tokeniser $\tokeniser$---always returns perfect models, i.e., models for which $\pthetatok(\subwords) \mathop{=} \ptok(\subwords)$.
    In this case, the causal effect of a subword $\subword$ is always zero, i.e., $\atemerge = 0$.
\end{restatable}
\begin{proof}
See the proof in \cref{app:no_effect_proof}.
\end{proof}

If the ideal effect is always zero, why bother estimating it? 
In practice, LMs are imperfect, and biases---both inductive and learned---may lead models to assign different probabilities to the same character-string under different tokenisations (all else held equal). Measuring this effect in practice is thus important. As an example, we show that at initialisation this causal effect is relatively large.

\begin{restatable}{theorem}{initialisationeffecttheorem}  \label{thm:initialisation_effect}
    Assume that at initialisation our language model outputs the uniform distribution $\pthetatok(\subword \!\mid\! \subwords_{<t}) \!=\! \frac{1}{|\vocab|}$.\footnote{See support for this in, e.g., \citet{chang-bergen-2022-word}.}
    In this case, the causal effect of a subword $\subword$ at initialisation is $\atemerge \gtrapprox \log |\vocab|$.
\end{restatable}
\begin{proof}
See the proof in \cref{app:initialisation_effect}.
\end{proof}

\section{Estimating Tokenisation Bias}

Now, we get to estimating the causal effect we defined; this is typically done in three steps. First, we define a \defn{causal estimand}, the theoretical quantity that represents the effect of interest; this is the value in \cref{def:ate_estimand}. Second, we rewrite this causal estimand in terms of quantities that we can actually observe, thus defining a \defn{statistical estimand}. Finally, we define an \defn{estimator}, a statistical procedure to approximate the statistical estimand.

Importantly, the causal estimand in \cref{eq:ate_estimand} is non-trivial to estimate, as it contains a counterfactual; this is exposed by a simple decomposition:
\begin{align}\label{eq:ate_decomposition}
     \ate = \underbrace{\expect_{\subword}\left[\expectpotentialoutcome[1]\right]}_{\circled{1}} - \underbrace{\expect_{\subword}\left[\expectpotentialoutcome[0]\right]}_{\circled{2}} 
\end{align}
Notably, we can observe the potential outcomes in expectation $\circled{1}$ only for subwords in the given tokeniser ($\subword \in \tokeniser$), while expectation $\circled{2}$ is observable only for subwords not in that tokeniser ($\subword \notin \tokeniser$).
Naively comparing a model's average performance in these two groups of in- and out-of-vocabulary items would return a biased estimate of the causal effect.
This is due to the way tokenisers are developed: subwords are not randomly included in the vocabulary, but carefully selected under some objective. As a subword's assigned treatment is not random, the two subgroups are not comparable, and there may thus be confounders.

Luckily, a popular class of tokenisation algorithms---which includes byte-pair encoding (\bpe, \citealp{sennrich-etal-2016-neural}) and WordPiece (\wordpiece, \citealp{schuster-nakajima-2012-voice})---select subwords iteratively, one-at-a-time, to compose their vocabularies.
This creates a ranking among subwords which deterministically defines their treatment assignment (i.e., whether $\subword \in \tokeniser$ or not); we will leverage the non-randomness in this treatment assignment to estimate our desired effect. We describe these methods next.

\subsection{Sequential Vocabulary Construction}
\label{sec:vocab_construction}

Let $\dataset = \{\chs[n]\}_{n=1}^{N}$ be a dataset of character-strings.
Bottom-up tokenisers typically define an \defn{objective function}, $\objectivefunc(\mergepair[1], \mergepair[2], \dataset)$, which measures the benefit of including a new subword $\mergepair[\texttt{new}] \!=\! \mergepair[1] \circ \mergepair[2]$ in the vocabulary.
Given a target vocabulary size $\vocabsize\!+\!\size{\alphabet}$, 
their algorithm initialises a vocabulary as $\vocab_0 \!=\! \alphabet$ and a dataset as $\dataset_0 \!=\! \dataset$.
For $\vocabsize$ iterations, it then selects the pair of subwords $\mergepair[1]_k, \mergepair[2]_k$ which maximises objective $\objectivefunc$ over the current dataset, adds it to the vocabulary $\vocab_{k} \!=\! \vocab_{k\!-\!1} \cup \{\mergepair[\texttt{new}]_k\}$, and applies it to $\dataset_{k} = \{\mergefunc{\mergestring{\mergepair[1]_k}{\mergepair[2]_k}}(\subwords) \mid \subwords \in \dataset_{k\!-\!1}\}$,
where $\mergefunc{}$ replaces consecutive occurrences of $\mergepair[1]_k,\mergepair[2]_k$ with $\mergepair[\texttt{new}]_k$.
The algorithm then returns $\vocab_{\vocabsize}$.

While the choice of objective function is arbitrary, most tokenisers employed by current LMs use either the objective from \bpe or \wordpiece:
\begin{align}
    \objectivebpe(\mergepair[1], \mergepair[2], \dataset_k) &\defeq 
    \countfunc(\subwordstring[false]{\mergepair[1], \mergepair[2]}, \dataset_k) \label{eq:bpe_obj}\\
     \objectivewordpiece(\mergepair[1], \mergepair[2], \dataset_k) &\defeq 
    \frac{
    \countfunc(\subwordstring[false]{\mergepair[1], \mergepair[2]}, \dataset_k)}{
    \countfunc(\subwordstring[false]{\mergepair[1]}, \dataset_k)\,
    \countfunc(\subwordstring[false]{\mergepair[2]}, \dataset_k)} \nonumber
\end{align}
where we use $\countfunc(\subwords, \dataset)$ to denote the number of occurrences of subword-string $\subwords$ in dataset $\dataset$.
At each iteration, \bpe's objective ($\objectivebpe$) selects the most frequent subword-pair in dataset $\dataset$, while \wordpiece's ($\objectivewordpiece$) selects the subword-pair with the largest pointwise mutual information in it.

\subsection{The Regression Discontinuity Design}
\label{sec:rd_design}

We define a principled estimator for tokenisation bias based on the regression discontinuity (RD) design \citep{thistlewaite1960regression}.\footnote{RD design is a quasi-experimental paradigm that exploits a deterministic relationship between an observed characteristic, known as the running variable, and the treatment assignment.
Unlike a true experiment, a quasi-experiment does not randomly assign participants to control and treatment groups.}
Let us consider running the algorithm above (\cref{sec:vocab_construction}) for $\vocabsize_+ \gg \vocabsize$ iterations and obtain a list of subwords:\looseness=-1
\begin{align}\label{eq:many_merges}
    [\mergepair[\texttt{new}]_1, \mergepair[\texttt{new}]_2, \smalldots, \mergepair[\texttt{new}]_{\vocabsize}, \smalldots, \mergepair[\texttt{new}]_{\vocabsize_+}]
\end{align}
The index $k \in \{1, \smalldots, \vocabsize_+\}$ of a subword in this sequence represents the iteration at which it would be selected by the tokenisation algorithm.
Index $k$ thus determines a natural ordering of how subwords are added to a bottom-up tokeniser.
Now, let us denote it as an index-valued random variable $\ordermerge$; this variable deterministically defines a subword's treatment assignment:
\begin{align}\label{eq:rd_treatment_assignment}
    \treatment_{\subword} &= \one\{\ordermerge \leq \vocabsize\}
\end{align}
In words, a subword is thus included in the vocabulary if its index $\ordermerge$ is smaller than $\vocabsize$.
In RD design, $\ordermerge$ is termed a \defn{running variable} and $\vocabsize$ is the \defn{cutoff} of this running variable.

For RD design to be applicable, this cutoff must be predetermined and exogenous---i.e., independent of potential outcomes---ensuring an unbiased treatment assignment.
This holds for bottom-up tokenisers, where vocabulary sizes are set beforehand.
Notably, we still cannot directly compare in- and out-of-vocabulary subwords, as they have different values of the running variable: by definition these groups have, respectively,  running variable values of $\runningvar\leq\cutoff$ or $\runningvar>\cutoff$.
This makes them systematically different and introduces bias.
The RD design addresses this issue by focusing on a \defn{local treatment effect} instead:
\begin{align}\label{eq:late_decomposition}
     \late 
     \defeq \expect_{\subword}\left[\expectpotentialoutcome[1] - \expectpotentialoutcome[0] \mid \runningvar \right]
\end{align}
Notably, $\late$ is thus a localised causal effect representing the effect expected for subwords added to the vocabulary at a specific index $\runningvar$.

We now leverage the cutoff point $\cutoff$ to estimate this local causal effect.
In short, since treatment assignment is deterministic and the cutoff is exogenous, this assignment mechanism is \enquote{as good as random} \citep{angrist-pischke-2015-mastering}.
If the outcome $\expectpotentialoutcome$ is a continuous function of the running variable, thus, it should transition smoothly across the cutoff in the absence of treatment, making any discontinuity at this point solely due to treatment.\footnote{Notably, our running variable $\runningvar$ is discrete, and thus not continuous by definition. In our experiments, we consider large enough windows around the cutoff and thus assume our variable is fine-grained enough to allow for such an estimation. Alternatively, we could have used other running variables for our experiments instead, such as, e.g., the objective function $\objectivefunc$.
We also note that, while the RD design typically assumes continuous running variables, it can be applied to discrete ones, such as the rank in our case \citep{cattaneo2024extensions}.}

\begin{figure*}[!t]
    \centering\small
    \includegraphics[trim={0 .5cm 0 0},clip,width=\linewidth]{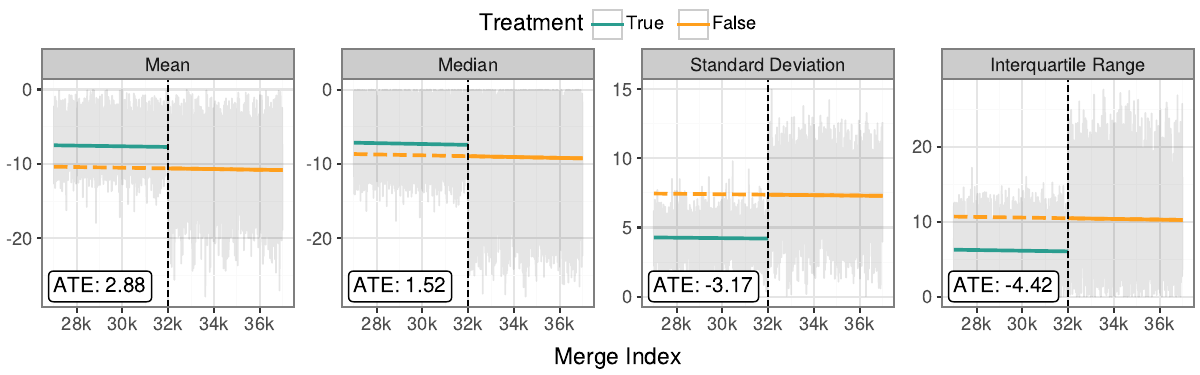}
    \vspace{-15pt}
    \caption{Tokenisation bias of $\subword\in\tokeniser$ for the last checkpoint of a model trained with \bpe tokeniser and $\vocabsize\mathop{=}\q{32}{\thousand}$.
    The x-axis shows the running variable $\ordermerge$, and the y-axis shows different outcome variables: mean, standard deviation, median, and interquartile range of a $\chssubword$'s log-probability across contexts.
    Subwords on the left-hand side of the cutoff are in the tokeniser. The dashed orange line indicates the estimated counterfactual outcome.
    }
    \label{fig:bpe32000_metrics}
    \vspace{-6pt}
\end{figure*}

\subsection{The Regression Discontinuity Estimand}

Before introducing a statistical estimand, we need to introduce the functional relationship between potential outcomes and the running variable.
We can write this formally as:
\begin{subequations} \label{eq:cef}
\begin{align}
    &\expect_{\subword}[\expectpotentialoutcome[0] \mid \runningvar] = \funcrel(\ordermerge) \label{eq:cef_control} \\
    &\expect_{\subword}[\expectpotentialoutcome[1] \mid \runningvar] = \funcrel(\ordermerge) + \aterunningvar \label{eq:cef_treat}
\end{align}
\end{subequations}
where \cref{eq:cef_treat} immediately follows from \cref{eq:late_decomposition}.
As typical in RD, we posit a smooth relationship between the running variable and potential outcomes.

\begin{assumption}[Continuity Assumption]\label{ass:continuity}
The conditional expectation of a potential outcome is continuous at the cutoff. That is, both limits exist:
\begin{subequations}
\begin{align}
    \lim_{\cutoffdelta \to \cutoff} &\expect_{\subword}[\expectpotentialoutcome[0] \mid \ordermerge\mathop{=}\cutoffdelta] \\
    \lim_{\cutoffdelta \to \cutoff} &\expect_{\subword}[\expectpotentialoutcome[1] \mid \ordermerge\mathop{=}\cutoffdelta]
\end{align}
\end{subequations}
\end{assumption}

Now, we note that \cref{eq:rd_treatment_assignment} implies that there is no value of the running variable for which we can \emph{observe} both potential outcomes, where we write an \defn{observed outcome} as:\footnote{Assuming consistency, i.e., assuming that a potential outcome under the observed treatment assignment equals the observed outcome \citep{cole-and-frangakis-2009-consistency}.}
\begin{align}\label{eq:observed_outcomes}
    \observedoutcome =
    \begin{cases}
        \expectpotentialoutcome[1] & \text{if } \treatment_{\subword}\mathop{=}1 \\
        \expectpotentialoutcome[0] & \text{otherwise}
    \end{cases}
\end{align}
Luckily, \citet{hahn-etal-2001-identification} show that the only required assumption to identify the causal estimand in \cref{eq:late_decomposition} is that \cref{ass:continuity} holds.
We can write a statistical estimand at the cutoff as:
\begin{align}\label{eq:rdestimand}
    \rdestimand = 
    \lim_{\cutoffdelta \to \cutoff^{-}} &\expect_{\subword}[\observedoutcome \mid \ordermerge\mathop{=}\cutoffdelta] \\
    &- \lim_{\cutoffdelta \to \cutoff^{+}} \expect_{\subword}[\observedoutcome \mid \ordermerge\mathop{=}\cutoffdelta] \nonumber
\end{align}
In words, the RD estimand is the discontinuity in the conditional expectation function at the cutoff.\footnote{
If we could estimate \cref{eq:rdestimand}'s expectations, we would have a principled (non-parametric) estimator for $\aterunningvar$.
Obtaining such estimates, however, is tricky \citep[chapter 6]{angrist2009mostly}.
First, working in a small neighbourhood of the cutoff means we would have little data.
Second, sample averages are biased when estimated on one side of a boundary.}

\subsection{The Regression Discontinuity Estimator}

In the limit of \cref{eq:rdestimand}, the functional terms $\funcrel(\ordermerge)$ in the observed effects cancel out.
However, we do not have a large number of samples exactly at the cutoff.
In practice, thus, RD is estimated by extrapolating over values of the running variable in a window around the cutoff; in this case, these functional terms $\funcrel(\ordermerge)$ do not cancel out.
To control for them, we first take a conditional expectation of $\observedoutcome$, which we write in terms of \cref{eq:cef}:
\begin{align}\label{eq:observed_cef}
    \expect_{\subword}[\observedoutcome \mid \runningvar] = \funcrel(\ordermerge) + \aterunningvar\,\treatment_{\subword}
\end{align}
We then assume a parametric form for $\funcrel$ (e.g., linear in $\runningvar$) and $\aterunningvar$ (e.g., a constant). 
Given a set of observed subwords $\observedset$ in a window around the cutoff, we obtain an estimator of $\aterunningvar$ by minimising the squared error of the functional fit of $\funcrel$ and $\aterunningvar$.

\begin{estimator}\label{defn:rd_estimator}
    Given a parametric function $\funcrel$ and a set of observed subwords around the cutoff $\observedset$, the \defn{regression discontinuity estimator} $\rdestimator$ is:
    \begin{align}\label{eq:rd_estimator}
        \argmin_{\funcrel, \rdestimand} \sum_{\subword \in \observedset}(\observedoutcome \!-\! \funcrel(\ordermerge) \!-\! \rdestimand\,\treatment_{\subword})^2
    \end{align}
    This is an unbiased estimator of the local effect $\aterunningvar$ at $\runningvar=\cutoff$ under \cref{ass:continuity} and assuming $\funcrel$ provides an adequate description of expected potential effects \citep{angrist2009mostly}.
\end{estimator}

Notably, the validity of an RD estimator hinges on the correct specification of $\funcrel$ and on the size of the window used to obtain the set $\observedset$.
On the one hand, a larger window allows for a bigger sample size but relies more on the correct estimation of $\funcrel$.
On the other hand, a smaller window depends less strongly on $\funcrel$ but may not have enough samples for precise statistical estimation.
Thus, RD requires careful tuning of the window size to control this bias-variance trade-off.
In \cref{fig:bpe32000_metrics}, we show an example of this method in practice.

\begin{figure}[!t]
    \centering\small
     \begin{subfigure}[b]{0.49\columnwidth}
        \includegraphics[width=\columnwidth]{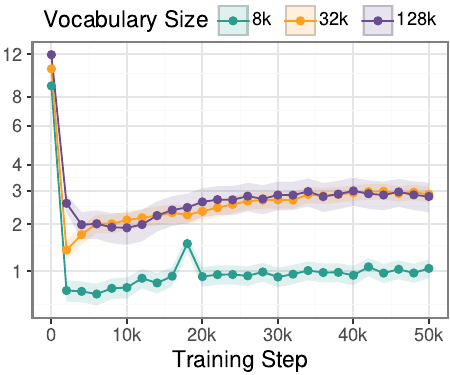}
     \end{subfigure}%
     \hfill
     \begin{subfigure}[b]{0.49\columnwidth}
        \includegraphics[width=\columnwidth]{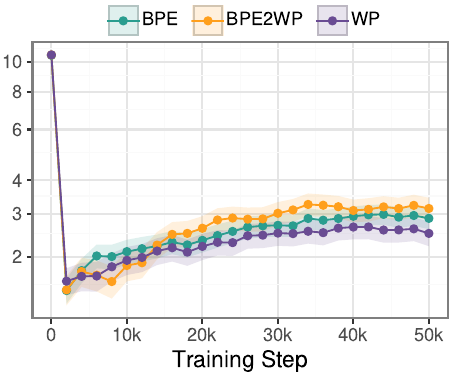}
     \end{subfigure}
     \vspace{-3pt}
    \caption{Tokenisation bias $\late$ of \q{57}{\million} across training for  (left) \bpe across vocabulary sizes \q{8}{\thousand}, \q{32}{\thousand}, and \q{128}{\thousand}; and for (right) vocabulary size \q{32}{\thousand} across tokenisers \bpe, \wordpiece, and \bpewp.}
    \label{fig:vocab_sizes_and_tokenisers}
    \vspace{-4pt}
\end{figure}

\section{Experimental Setup}

We conduct experiments on a suite of Llama-style transformer models of varying scales, trained with different tokenisation schemes and vocabulary sizes.\footnote{
Since our method requires knowing which subwords would be added after the cutoff, we could not use open-source LMs, as this information is typically not published even for open-source models like Pythia \citep{biderman-etal-2023-pythia} and OLMo \citep{groeneveld-etal-2024-olmo}. Further, their tokeniser construction details are often unreproducible. We therefore train new tokenisers and models from scratch.} 
More details in \cref{app:implementation_details}.

\paragraph{Models and Tokenisers.}
Unless otherwise specified, we run our experiments with a \bpe tokeniser with vocabulary size $\vocabsize\mathop{=}\q{32}{\thousand}$, and a Llama \citep{touvron2023llama} LM with \q{57}{\million} non-embedding parameters trained on the \minipile \citep{kaddour2023minipile} dataset for \q{50}{\thousand} steps.
To explore how tokenisation bias changes as a function of: (i) model size, we further train Llama variants with approximately \q{340}{\million}, and \q{850}{\million} non-embedding parameters;
(ii) vocabulary size, we experiment with \bpe's with $\vocabsize$ equal \q{8}{\thousand}, \q{32}{\thousand}, and \q{128}{\thousand};
(iii) tokenisation algorithm, we experiment with a standard \wordpiece tokeniser, and a tokeniser using \wordpiece's tokenisation function but \bpe's vocabulary, which we call \bpewp;
(iv) embedding tying and training dataset, we train a \q{100}{\million} parameter model with and without tied input-output embeddings on \q{20}{\billion} tokens from the \fineweb corpus \citep{penedo2024the}.

\paragraph{Evaluation and Data Collection.}
All models are evaluated on the \minipile validation set. 
To compute $\observedoutcome$, we collect subwords' log-probabilities and aggregate them across the contexts in which they appear.
For our regression discontinuity estimation, we focus on a window of \q{5}{\thousand} subwords around the tokeniser’s vocabulary cutoff.

\section{Results}

In this section, we report the estimated tokenisation bias $\late$ in our models.
\cref{fig:bpe32000_metrics} (left) presents tokenisation bias in \bpe with default parameters.
We estimate this bias to be $\float[2]{2.88}$ nats, with treated character-strings $\chssubword$ having a log-probability of \float[2]{-7.720603} and untreated ones \float[2]{-10.599314}.
This means a $\chssubword$ can be assigned roughly \integer{17}
times less probability due solely to tokenisation, implying LMs are highly susceptible to tokenisation bias.
We next examine this effect in detail.

\paragraph{Effect Across Training.}  
\cref{fig:vocab_sizes_and_tokenisers} (left; orange line) shows the causal effect evolving through training. 
As predicted (\cref{thm:initialisation_effect}), this effect is large at the start of training. 
Notably, it drops sharply by step \q{2}{\thousand} and then slowly grows again.
This is counterintuitive. 
Since an ideal model would be unaffected by tokenisation (\cref{thm:no_effect}), one would expect a decline in $\late$ across training.
Instead, our results suggest models become more biased as they improve.

\begin{figure*}[!t]
    \centering
    \begin{subfigure}[b]{0.25\linewidth}
    \includegraphics[width=\columnwidth]{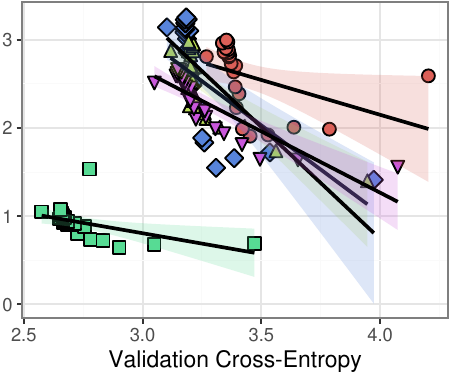}
    \end{subfigure}%
    \hfill
    \begin{subfigure}[b]{0.25\linewidth}
    \includegraphics[width=\columnwidth]{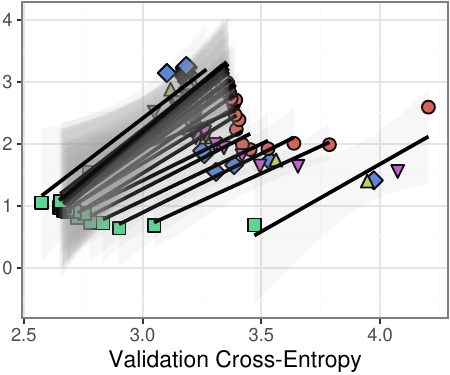}
    \end{subfigure}%
    \hfill
    \begin{subfigure}[b]{0.25\linewidth}
    \includegraphics[width=\columnwidth]{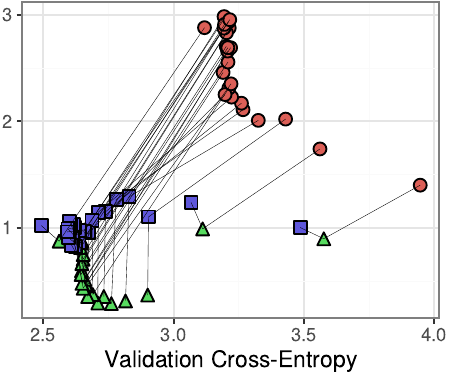}
    \end{subfigure}%
    \hfill
    \begin{subfigure}[b]{0.25\linewidth}
    \includegraphics[width=\columnwidth]{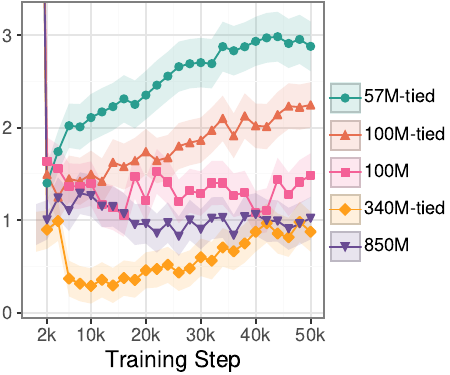}
    \end{subfigure}%
    \caption{(Left to Center-right) Tokenisation bias $\late$ vs.\ validation cross-entropy of trained models. Each point represents a training step of a model trained with a specific tokeniser.
    (Left) Linear fits per tokeniser and across training steps, size fixed to \q{57}{\million};
    (Center-left) Linear fits per training step and across tokenisers, size fixed to \q{57}{\million};
    (Center-right) Linear fits per training step and across model sizes, tokeniser fixed as \bpe with $\vocabsize=\q{32}{\thousand}$.
    (Right-most) Tokenisation bias $\late$ across training of models with different sizes and \bpe tokeniser with $\vocabsize=\q{32}{\thousand}$. }
    \label{fig:scaling_plot}
    \vspace{-8pt}
\end{figure*}

\vspace{-1pt}
\paragraph{Comparison Across Vocabulary Sizes.}  
\cref{fig:vocab_sizes_and_tokenisers} (left) also shows how tokenisation bias varies with vocabulary size ($\vocabsize \in \{\q{8}{\thousand}, \q{32}{\thousand}, \q{128}{\thousand}\}$).
We find weaker biases when using smaller vocabularies. 
This might be because for smaller $\cutoff$, the character-strings $\chssubword$ near the cutoff will be more frequent; untreated strings $\chssubword$ will thus see more training, which potentially reduces tokenisation bias.
Beyond a critical size, the effect stabilises, with \q{128}{\thousand} and \q{32}{\thousand} showing similar biases despite a large difference in vocabulary size (see also \cref{fig:bpe_vocabsize_metrics}, \cref{app:addtional_plots}).

\vspace{-1pt}
\paragraph{Comparison Across Tokenisers.}
\cref{fig:vocab_sizes_and_tokenisers} (right) shows how the causal effect changes as a function of the used tokenisation and objective functions, displaying results for \bpe (with $\tokbup, \objectivebpe$), \bpewp (with  $\tokbup, \objectivewordpiece$), and \wordpiece (with  $\toklongest, \objectivewordpiece$).
In this figure, we see that tokenisation bias seems consistent across these tokenisers (see also \cref{fig:alltokenisers32000_metrics}, \cref{app:addtional_plots}).

\vspace{-1pt}
\paragraph{Comparison Across Model Sizes.}
\Cref{fig:scaling_plot} (right) shows how tokenisation bias changes as a function of model size.
In this figure, we see that bias is smaller on larger (e.g., $\q{340}{\million}$ parameters) than in small models (e.g., $\q{57}{\million}$).
Among our largest models, however (with either $\q{340}{\million}$ or $\q{850}{\million}$ parameters), we see little difference in tokenisation bias.
Further, even in the largest models, bias does not seem to decrease across training.
Together, these results suggest that tokenisation bias may be a persisting property of LMs.

\vspace{-1pt}
\paragraph{Tokenisation Bias  vs.\ Model Quality.}
We next explore in more detail the previous observation that tokenisation bias grows across training (\cref{fig:vocab_sizes_and_tokenisers})---suggesting an inverse-scaling effect where tokenisation bias strengthens as models improve \citep{mckenzie2023inverse}.
\cref{fig:scaling_plot} plots the estimated $\late$ against model quality (quantified as cross-entropy). 
As can be seen, two opposing trends emerge: across training, better models show larger causal effects (shown on the left); across tokenisers with different vocabulary sizes, better models show smaller effects (shown on the center-left).
Further, this relationship seems to change across model scales (shown on the center-right).
This highlights that tokenisation bias is a complex phenomenon, which can either improve or worsen as models get better.

\vspace{-1pt}
\paragraph{Other Causal Effects.}
Our framework also allows us to study tokenisation-related biases beyond average model performance. 
Here, we examine the effect of tokenisation bias on the stability of model outputs across contexts.
To this end, we define new potential outcome variables:
\begin{align}\label{eq:other_potential_outcomes}
    \otherpotentialoutcome[\treatment] \defeq \generalaggfunction_{\chs[<t]}\left[
    \log \pthetatok(\chssubword \mid \chs[<t])
    \right]
\end{align}
where we consider standard deviation, median, and interquartile range as $\generalaggfunction$.
Results are in \cref{fig:bpe32000_metrics} (also in \cref{fig:bpe_vocabsize_metrics,fig:alltokenisers32000_metrics}, \cref{app:addtional_plots}). 
In particular, standard deviation results show that whether a subword is present in an LM's vocabulary significantly affects output stability: $\chssubword$ of in-vocabulary subwords exhibit far less variation in log-probability across contexts than out-of-vocabulary ones.
This again underscores the strong influence of tokenisation on model behaviour.

\begin{figure}[!t]
    \centering
    \includegraphics[width=\columnwidth]{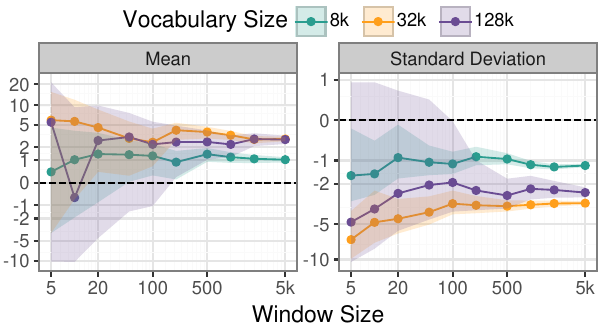}\vspace{-5pt}
    \caption{Estimated effect $\rdestimator$ (y-axis) vs.\ window size used to estimate it (x-axis). Results for fully trained model with \bpe and $\cutoff\in\{\q{8}{\thousand}, \q{32}{\thousand}, \q{128}{\thousand}\}$. Shaded regions correspond to standard errors.
    }%
    \label{fig:bpe_windowsizes_subset}
    \vspace{-8pt}
\end{figure}

\paragraph{Robustness of the RD Estimation.}
Finally, we analyse the robustness of our estimates with respect to: (i) window size used to select $\observedset$, and (ii) functional form of $\funcrel$ used for the regression fit.
First, \cref{fig:bpe_windowsizes_subset} shows that the estimated effect is unstable for window sizes smaller than \integer{500}, but stabilises when using at least \q{1}{\thousand} subwords on both sides of the cutoff (see also \cref{fig:bpe_windowsizes}, \cref{app:addtional_plots}).
Second, \cref{fig:loess_regression_model} (in \cref{app:addtional_plots}, due to space constraints) shows that more flexible specifications of $\funcrel$ return similar estimates to our linear choice, though fitting noise.

\section{Tokenisation Bias in Context}
\label{sec:discussion}

We now place our findings in the broader context of existing work on tokenisation, highlighting theoretical implications, potential links to known biases, and opportunities for improved fairness.

\subsection{Implications of Tokenisation Bias}

\paragraph{Theoretical Underpinnings of Tokenisation.}
Our results align with recent theoretical work on the role of tokenisation in language modelling.
\citet{Rajaraman_2024} show that subword vocabularies endow Transformers with inductive biases suited to modelling higher-order Markov sources, while character-level models lag behind. This is supported by our finding that representing a character-string as two subwords instead of one reliably lowers its log-probability.

\paragraph{Length Bias.}
Tokenisation bias may help explain the well-documented length bias in LMs \citep{murray-chiang-2018-correcting, stahlberg-byrne-2019-nmt}, where longer subword sequences are less likely to be generated.
We find that character-strings tokenised as multiple subwords are systematically assigned lower probabilities, which may cause models to prefer generating shorter sequences. 

\paragraph{Multilingual Fairness.}
Tokenisation can also contribute to performance disparities across languages. \citet{petrov2023token_unfairness} show that under typical multilingual tokenisers, lower-resource languages tend to have longer tokenisations. 
While originally noted for its cost implications, we show this may also hurt performance as, \textit{ceteris paribus}, (i) longer subword-strings receive lower probability due to tokenisation bias and, thus, (ii) they may reduce the model’s likelihood of generating low-resource language content, reinforcing data sparsity effects.
These effects echo \citeposs{rust-etal-2021-good} findings, which show that language-specific tokenisers close much of the performance gap between monolingual and multilingual models, highlighting tokenisation as a key factor in multilingual performance.

\subsection{Going Forward: Uses for \texorpdfstring{$\late$}{local-ATE}  }

This paper provides new empirical evidence that tokenisation choices bias model predictions.
Prior studies have largely focused on observational evidence of these effects, often through adversarial prompts or indirect correlates \citep{chai-etal-2024-tokenization, wang2025tokenizationmattersdegradinglarge}.
Our estimator, by contrast, enables a more principled study of how vocabulary construction choices affect a LM's behaviour. We sketch two concrete use cases for $\ate$ below.

\paragraph{Measuring Impact on Lexical Generalisation.}
Adding subwords to the vocabulary causes models to treat the corresponding character span as an indivisible unit. This can be beneficial (e.g., in English, \subwordstring{pl} and \subwordstring{ay} do not have meanings on their own, but \subwordstring{play} does), but may hinder generalisation across morphological variants (e.g., \subwordstring{play} and \subwordstring{plays}). Tokenisation may thus impact model generalisation across orthographically similar lexical items.
Prior work already hints at this tradeoff: \citet{schafer-etal-2024-effect} show that duplicated entries in a vocabulary may hurt generalisation, while \citet{Toraman_2023} and \citet{schmidt-etal-2024-tokenization} show that vocabulary expansion helps only under morphology-aware tokenisation. 
These results suggest that adding too many subwords to the vocabulary can hurt lexical generalisation, a hypothesis our estimator could be adapted to test.

\paragraph{Optimising Tokenisation.}
Finding an optimal tokeniser is NP-complete \citep{whittington2024tokenisation, kozma-etal-2024-theoretical}. 
Further, existing heuristic metrics for tokeniser selection, such as subword frequency thresholds \citep{gowda-may-2020-finding} or R\'enyi efficiency \citep{zouhar-etal-2023-tokenization}, are often poorly correlated with downstream model performance \citep{schmidt-etal-2024-tokenization}. 
Our causal estimator offers a more grounded alternative.
If a practitioner aims to optimise held-out perplexity, a positive $\ate$ would signal that expanding the vocabulary (i.e., including more subwords) could help. Conversely, a negligible or negative $\ate$ might justify shrinking the vocabulary to gain efficiency without hurting performance. This moves us beyond heuristic selection toward a more systematic, model-level approach to tokeniser design.

\section{Conclusion}

We study a phenomenon we call \defn{tokenisation bias}: the extent to which a model’s outputs are affected by whether a subword appears in its tokeniser’s vocabulary. 
We propose a new method to measure it without re-training the model. 
We empirically show that character-strings tokenised as a single subword receive significantly more probability than when split, and that this effect intensifies over training.

\section*{Limitations}

This work estimates tokenisation bias: the extent to which a model’s output depends on whether a subword appears in its tokeniser’s vocabulary. 
Unfortunately, due to the costs associated with training large LMs, most of our experiments focused on relatively small models trained in a single language (English).
Investigating whether other model architectures, training procedures, and natural languages result in similar causal effects would be important to strengthen our conclusions.
Moreover, our method estimates a local causal effect (i.e., the effect for subwords in a window around the cutoff).
However, our results suggest that the estimated effect can be extrapolated further beyond the window size, as our regression estimates show an almost flat trend with respect to the running variable.

\section*{Acknowledgments}
\setlength{\intextsep}{0pt}
\setlength{\columnsep}{8pt}
\begin{wrapfigure}{l}{0.45\columnwidth}
    \includegraphics[width=0.45\columnwidth]{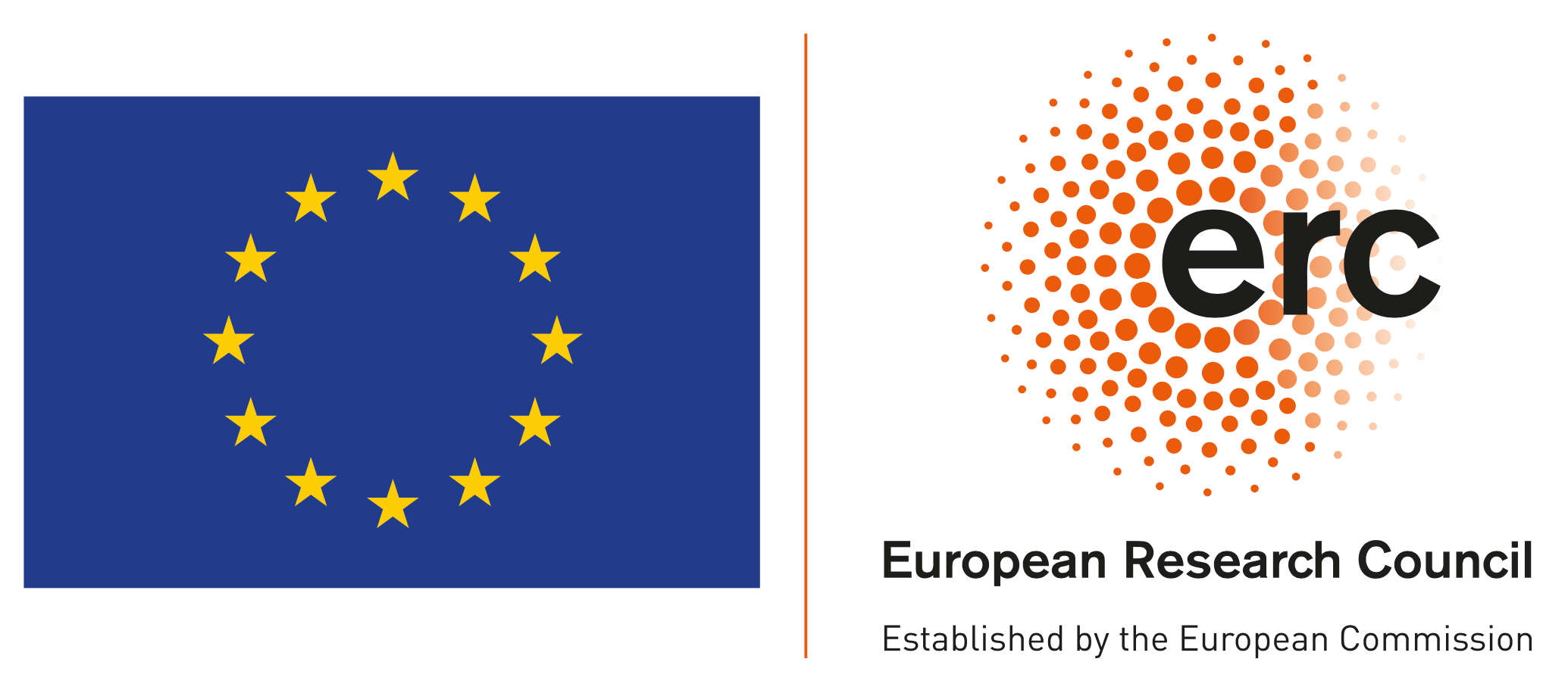}
\end{wrapfigure}
Pietro and Andreas received funding from the European Research Council (ERC) under the European Union’s Horizon 2020 Research and Innovation programme grant AVeriTeC (Grant agreement No. 865958).
We thank the anonymous reviewers for their helpful questions and comments.
We also thank Gregor Bachmann for his feedback on an earlier version of this paper.

\bibliography{biblio.bib}

\begin{thebibliography}{48}
\providecommand{\natexlab}[1]{#1}

\bibitem[{Ali et~al.(2024)Ali, Fromm, Thellmann, Rutmann, L{\"u}bbering, Leveling, Klug, Ebert, Doll, Buschhoff, Jain, Weber, Jurkschat, Abdelwahab, John, Ortiz~Suarez, Ostendorff, Weinbach, Sifa, Kesselheim, and Flores-Herr}]{ali-etal-2024-tokeniser}
Mehdi Ali, Michael Fromm, Klaudia Thellmann, Richard Rutmann, Max L{\"u}bbering, Johannes Leveling, Katrin Klug, Jan Ebert, Niclas Doll, Jasper Buschhoff, Charvi Jain, Alexander Weber, Lena Jurkschat, Hammam Abdelwahab, Chelsea John, Pedro Ortiz~Suarez, Malte Ostendorff, Samuel Weinbach, Rafet Sifa, Stefan Kesselheim, and Nicolas Flores-Herr. 2024.
\newblock \href {https://doi.org/10.18653/v1/2024.findings-naacl.247} {Tokenizer choice for {LLM} training: {N}egligible or crucial?}
\newblock In \emph{Findings of the Association for Computational Linguistics: NAACL 2024}, pages 3907--3924, Mexico City, Mexico. Association for Computational Linguistics.

\bibitem[{Angrist and Pischke(2015)}]{angrist-pischke-2015-mastering}
Joshua Angrist and J\"{o}rn-Steffen Pischke. 2015.
\newblock \href {https://www.masteringmetrics.com} {\emph{Mastering 'Metrics: The Path from Cause to Effect}}.
\newblock Princeton University Press.

\bibitem[{Angrist and Pischke(2009)}]{angrist2009mostly}
Joshua~D. Angrist and J{\"o}rn-Steffen Pischke. 2009.
\newblock \href {https://www.dsecoaching.com/pdf/2008%20Angrist%20Pischke%20MostlyHarmlessEconometrics.pdf} {\emph{Mostly harmless econometrics: {A}n empiricist's companion}}.
\newblock Princeton university press.

\bibitem[{Biderman et~al.(2022)Biderman, Bicheno, and Gao}]{biderman-etal-2022-datasheet}
Stella Biderman, Kieran Bicheno, and Leo Gao. 2022.
\newblock \href {https://doi.org/10.48550/arXiv.2201.07311} {Datasheet for the {{Pile}}}.
\newblock \emph{arXiv preprint 2201.07311}.

\bibitem[{Biderman et~al.(2023)Biderman, Schoelkopf, Anthony, Bradley, O'Brien, Hallahan, Khan, Purohit, Prashanth, Raff, Skowron, Sutawika, and Van Der~Wal}]{biderman-etal-2023-pythia}
Stella Biderman, Hailey Schoelkopf, Quentin Anthony, Herbie Bradley, Kyle O'Brien, Eric Hallahan, Mohammad~Aflah Khan, Shivanshu Purohit, USVSN~Sai Prashanth, Edward Raff, Aviya Skowron, Lintang Sutawika, and Oskar Van Der~Wal. 2023.
\newblock \href {https://proceedings.mlr.press/v202/biderman23a/biderman23a.pdf} {Pythia: {A} suite for analyzing large language models across training and scaling}.
\newblock In \emph{Proceedings of the 40th International Conference on Machine Learning}, ICML'23.

\bibitem[{Cattaneo et~al.(2024)Cattaneo, Idrobo, and Titiunik}]{cattaneo2024extensions}
Matias~D. Cattaneo, Nicolas Idrobo, and Roc\'{\i}o Titiunik. 2024.
\newblock \href {https://doi.org/10.1017/9781009441896} {\emph{A Practical Introduction to Regression Discontinuity Designs: {E}xtensions}}.
\newblock Cambridge University Press.

\bibitem[{Chai et~al.(2024)Chai, Fang, Peng, and Li}]{chai-etal-2024-tokenization}
Yekun Chai, Yewei Fang, Qiwei Peng, and Xuhong Li. 2024.
\newblock \href {https://doi.org/10.18653/v1/2024.findings-emnlp.86} {Tokenization falling short: On subword robustness in large language models}.
\newblock In \emph{Findings of the Association for Computational Linguistics: EMNLP 2024}, pages 1582--1599, Miami, Florida, USA. Association for Computational Linguistics.

\bibitem[{Chang and Bergen(2022)}]{chang-bergen-2022-word}
Tyler~A. Chang and Benjamin~K. Bergen. 2022.
\newblock \href {https://doi.org/10.1162/tacl_a_00444} {Word acquisition in neural language models}.
\newblock \emph{Transactions of the Association for Computational Linguistics}, 10:1--16.

\bibitem[{Cole and Frangakis(2009)}]{cole-and-frangakis-2009-consistency}
Stephen~R. Cole and Constantine~E. Frangakis. 2009.
\newblock \href {https://journals.lww.com/epidem/fulltext/2009/01000/the\%5Fconsistency\%5Fstatement\%5Fin\%5Fcausal\%5Finference\%5F\%5Fa.3.aspx} {The consistency statement in causal inference: {A} definition or an assumption?}
\newblock \emph{Epidemiology}, 20(1).

\bibitem[{Du et~al.(2023)Du, Torroba~Hennigen, Pimentel, Meister, Eisner, and Cotterell}]{du-etal-2023-measure}
Li~Du, Lucas Torroba~Hennigen, Tiago Pimentel, Clara Meister, Jason Eisner, and Ryan Cotterell. 2023.
\newblock \href {https://doi.org/10.18653/v1/2023.acl-long.543} {A measure-theoretic characterization of tight language models}.
\newblock In \emph{Proceedings of the 61st Annual Meeting of the Association for Computational Linguistics (Volume 1: Long Papers)}, pages 9744--9770, Toronto, Canada. Association for Computational Linguistics.

\bibitem[{Gage(1994)}]{Gage1994}
Philip Gage. 1994.
\newblock \href {https://dl.acm.org/doi/10.5555/177910.177914} {A new algorithm for data compression}.
\newblock \emph{C Users Journal}, 12(2):23–38.

\bibitem[{Gao et~al.(2020)Gao, Biderman, Black, Golding, Hoppe, Foster, Phang, He, Thite, Nabeshima, Presser, and Leahy}]{gao-etal-2020-pile}
Leo Gao, Stella Biderman, Sid Black, Laurence Golding, Travis Hoppe, Charles Foster, Jason Phang, Horace He, Anish Thite, Noa Nabeshima, Shawn Presser, and Connor Leahy. 2020.
\newblock \href {http://arxiv.org/abs/2101.00027} {The {{Pile}}: {{An 800GB}} dataset of diverse text for language modeling}.
\newblock \emph{arXiv preprint 2101.00027}.

\bibitem[{Gowda and May(2020)}]{gowda-may-2020-finding}
Thamme Gowda and Jonathan May. 2020.
\newblock \href {https://doi.org/10.18653/v1/2020.findings-emnlp.352} {Finding the optimal vocabulary size for neural machine translation}.
\newblock In \emph{Findings of the Association for Computational Linguistics: EMNLP 2020}, pages 3955--3964, Online. Association for Computational Linguistics.

\bibitem[{Groeneveld et~al.(2024)Groeneveld, Beltagy, Walsh, Bhagia, Kinney, Tafjord, Jha, Ivison, Magnusson, Wang, Arora, Atkinson, Authur, Chandu, Cohan, Dumas, Elazar, Gu, Hessel, Khot, Merrill, Morrison, Muennighoff, Naik, Nam, Peters, Pyatkin, Ravichander, Schwenk, Shah, Smith, Strubell, Subramani, Wortsman, Dasigi, Lambert, Richardson, Zettlemoyer, Dodge, Lo, Soldaini, Smith, and Hajishirzi}]{groeneveld-etal-2024-olmo}
Dirk Groeneveld, Iz~Beltagy, Evan Walsh, Akshita Bhagia, Rodney Kinney, Oyvind Tafjord, Ananya Jha, Hamish Ivison, Ian Magnusson, Yizhong Wang, Shane Arora, David Atkinson, Russell Authur, Khyathi Chandu, Arman Cohan, Jennifer Dumas, Yanai Elazar, Yuling Gu, Jack Hessel, Tushar Khot, William Merrill, Jacob Morrison, Niklas Muennighoff, Aakanksha Naik, Crystal Nam, Matthew Peters, Valentina Pyatkin, Abhilasha Ravichander, Dustin Schwenk, Saurabh Shah, William Smith, Emma Strubell, Nishant Subramani, Mitchell Wortsman, Pradeep Dasigi, Nathan Lambert, Kyle Richardson, Luke Zettlemoyer, Jesse Dodge, Kyle Lo, Luca Soldaini, Noah Smith, and Hannaneh Hajishirzi. 2024.
\newblock \href {https://doi.org/10.18653/v1/2024.acl-long.841} {{OLM}o: {A}ccelerating the science of language models}.
\newblock In \emph{Proceedings of the 62nd Annual Meeting of the Association for Computational Linguistics (Volume 1: Long Papers)}, pages 15789--15809, Bangkok, Thailand. Association for Computational Linguistics.

\bibitem[{Hahn et~al.(2001)Hahn, Todd, and der Klaauw}]{hahn-etal-2001-identification}
Jinyong Hahn, Petra Todd, and Wilbert~Van der Klaauw. 2001.
\newblock \href {http://www.jstor.org/stable/2692190} {Identification and estimation of treatment effects with a regression-discontinuity design}.
\newblock \emph{Econometrica}, 69(1):201--209.

\bibitem[{Hu et~al.(2024)Hu, Tu, Han, Cui, He, Zhao, Long, Zheng, Fang, Huang, Zhang, Thai, Wang, Yao, Zhao, Zhou, Cai, Zhai, Ding, Jia, Zeng, dahai li, Liu, and Sun}]{hu2024minicpm}
Shengding Hu, Yuge Tu, Xu~Han, Ganqu Cui, Chaoqun He, Weilin Zhao, Xiang Long, Zhi Zheng, Yewei Fang, Yuxiang Huang, Xinrong Zhang, Zhen~Leng Thai, Chongyi Wang, Yuan Yao, Chenyang Zhao, Jie Zhou, Jie Cai, Zhongwu Zhai, Ning Ding, Chao Jia, Guoyang Zeng, dahai li, Zhiyuan Liu, and Maosong Sun. 2024.
\newblock \href {https://openreview.net/forum?id=3X2L2TFr0f} {Mini{CPM}: {U}nveiling the potential of small language models with scalable training strategies}.
\newblock In \emph{First Conference on Language Modeling}.

\bibitem[{Kaddour(2023)}]{kaddour2023minipile}
Jean Kaddour. 2023.
\newblock \href {https://arxiv.org/abs/2304.08442} {The {M}ini{P}ile challenge for data-efficient language models}.
\newblock \emph{arXiv preprint 2304.08442}.

\bibitem[{Kingma and Ba(2015)}]{adam}
Diederik~P. Kingma and Jimmy Ba. 2015.
\newblock \href {http://arxiv.org/abs/1412.6980} {Adam: {A} method for stochastic optimization}.
\newblock In \emph{3rd International Conference on Learning Representations}.

\bibitem[{Kozma and Voderholzer(2024)}]{kozma-etal-2024-theoretical}
L\'{a}szl\'{o} Kozma and Johannes Voderholzer. 2024.
\newblock \href {https://arxiv.org/abs/2411.08671} {Theoretical analysis of byte-pair encoding}.
\newblock \emph{arXiv preprint 2411.08671}.

\bibitem[{Loshchilov and Hutter(2019)}]{loshchilov2018decoupled}
Ilya Loshchilov and Frank Hutter. 2019.
\newblock \href {https://openreview.net/forum?id=Bkg6RiCqY7} {Decoupled weight decay regularization}.
\newblock In \emph{International Conference on Learning Representations}.

\bibitem[{Lu et~al.(2018)Lu, Sadiq, Feaster, and Ishwaran}]{lu-etal-2018-estimating}
Min Lu, Saad Sadiq, Daniel~J Feaster, and Hemant Ishwaran. 2018.
\newblock \href {https://doi.org/10.1080/10618600.2017.1356325} {Estimating individual treatment effect in observational data using random forest methods}.
\newblock \emph{Journal of Computational and Graphical Statistics}, 27(1):209--219.

\bibitem[{McKenzie et~al.(2023)McKenzie, Lyzhov, Pieler, Parrish, Mueller, Prabhu, McLean, Shen, Cavanagh, Gritsevskiy, Kauffman, Kirtland, Zhou, Zhang, Huang, Wurgaft, Weiss, Ross, Recchia, Liu, Liu, Tseng, Korbak, Kim, Bowman, and Perez}]{mckenzie2023inverse}
Ian~R. McKenzie, Alexander Lyzhov, Michael~Martin Pieler, Alicia Parrish, Aaron Mueller, Ameya Prabhu, Euan McLean, Xudong Shen, Joe Cavanagh, Andrew~George Gritsevskiy, Derik Kauffman, Aaron~T. Kirtland, Zhengping Zhou, Yuhui Zhang, Sicong Huang, Daniel Wurgaft, Max Weiss, Alexis Ross, Gabriel Recchia, Alisa Liu, Jiacheng Liu, Tom Tseng, Tomasz Korbak, Najoung Kim, Samuel~R. Bowman, and Ethan Perez. 2023.
\newblock \href {https://openreview.net/forum?id=DwgRm72GQF} {Inverse scaling: {When} bigger isn't better}.
\newblock \emph{Transactions on Machine Learning Research}.
\newblock Featured Certification.

\bibitem[{Murray and Chiang(2018)}]{murray-chiang-2018-correcting}
Kenton Murray and David Chiang. 2018.
\newblock \href {https://doi.org/10.18653/v1/W18-6322} {Correcting length bias in neural machine translation}.
\newblock In \emph{Proceedings of the Third Conference on Machine Translation: Research Papers}, pages 212--223, Brussels, Belgium. Association for Computational Linguistics.

\bibitem[{Paszke et~al.(2019)Paszke, Gross, Massa, Lerer, Bradbury, Chanan, Killeen, Lin, Gimelshein, Antiga, Desmaison, Kopf, Yang, DeVito, Raison, Tejani, Chilamkurthy, Steiner, Fang, Bai, and Chintala}]{paszke-etal-2019-pytorch}
Adam Paszke, Sam Gross, Francisco Massa, Adam Lerer, James Bradbury, Gregory Chanan, Trevor Killeen, Zeming Lin, Natalia Gimelshein, Luca Antiga, Alban Desmaison, Andreas Kopf, Edward Yang, Zachary DeVito, Martin Raison, Alykhan Tejani, Sasank Chilamkurthy, Benoit Steiner, Lu~Fang, Junjie Bai, and Soumith Chintala. 2019.
\newblock \href {https://proceedings.neurips.cc/paper\%5Ffiles/paper/2019/hash/bdbca288fee7f92f2bfa9f7012727740-Abstract.html} {{{PyTorch}}: {A}n imperative style, high-performance deep learning library}.
\newblock In \emph{Advances in {{Neural Information Processing Systems}}}, volume~32. {Curran Associates, Inc.}

\bibitem[{Penedo et~al.(2024)Penedo, Kydl{\'\i}{\v{c}}ek, allal, Lozhkov, Mitchell, Raffel, Werra, and Wolf}]{penedo2024the}
Guilherme Penedo, Hynek Kydl{\'\i}{\v{c}}ek, Loubna~Ben allal, Anton Lozhkov, Margaret Mitchell, Colin Raffel, Leandro~Von Werra, and Thomas Wolf. 2024.
\newblock \href {https://openreview.net/forum?id=n6SCkn2QaG} {The {FineWeb} datasets: {D}ecanting the web for the finest text data at scale}.
\newblock In \emph{The Thirty-eight Conference on Neural Information Processing Systems Datasets and Benchmarks Track}.

\bibitem[{Petrov et~al.(2023)Petrov, La~Malfa, H.~S.~Torr, and Bibi}]{petrov2023token_unfairness}
Aleksandar Petrov, Emanuele La~Malfa, Philip H.~S.~Torr, and Adel Bibi. 2023.
\newblock \href {https://arxiv.org/abs/2305.15425} {Language model tokenizers introduce unfairness between languages}.
\newblock In \emph{Advances in Neural Information Processing Systems}.

\bibitem[{Phan et~al.(2025)Phan, Amos, Gat, Havasi, Muckley, and Ullrich}]{phan-etal-2025-exact}
Buu Phan, Brandon Amos, Itai Gat, Marton Havasi, Matthew~J. Muckley, and Karen Ullrich. 2025.
\newblock \href {https://openreview.net/forum?id=zGej22CBnS} {Exact byte-level probabilities from tokenized language models for {FIM}-tasks and model ensembles}.
\newblock In \emph{The Thirteenth International Conference on Learning Representations}.

\bibitem[{Phan et~al.(2024)Phan, Havasi, Muckley, and Ullrich}]{phan-etal-2024-understanding}
Buu Phan, Marton Havasi, Matthew~J. Muckley, and Karen Ullrich. 2024.
\newblock \href {https://openreview.net/forum?id=OqfdrBj1y1} {Understanding and mitigating tokenization bias in language models}.
\newblock In \emph{ICML 2024 Workshop on Theoretical Foundations of Foundation Models}.

\bibitem[{Pimentel and Meister(2024)}]{pimentel-meister-2024-compute}
Tiago Pimentel and Clara Meister. 2024.
\newblock \href {https://doi.org/10.18653/v1/2024.emnlp-main.1020} {How to compute the probability of a word}.
\newblock In \emph{Proceedings of the 2024 Conference on Empirical Methods in Natural Language Processing}, pages 18358--18375, Miami, Florida, USA. Association for Computational Linguistics.

\bibitem[{Rajaraman et~al.(2024)Rajaraman, Jiao, and Ramchandran}]{Rajaraman_2024}
Nived Rajaraman, Jiantao Jiao, and Kannan Ramchandran. 2024.
\newblock \href {https://proceedings.neurips.cc/paper\%5Ffiles/paper/2024/file/724afcaae4ae92a9220a077ffe80088d-Paper-Conference.pdf} {An analysis of tokenization: Transformers under {M}arkov data}.
\newblock In \emph{Advances in Neural Information Processing Systems}, volume~37, pages 62503--62556. Curran Associates, Inc.

\bibitem[{Rubin(1974)}]{rubin-1974-estimating}
Donald~B. Rubin. 1974.
\newblock \href {https://psycnet.apa.org/record/1975-06502-001?doi=1} {Estimating causal effects of treatments in randomized and nonrandomized studies}.
\newblock \emph{Journal of Educational Psychology}, 66(5):688--701.

\bibitem[{Rubin(2005)}]{rubin-2005-causal}
Donald~B. Rubin. 2005.
\newblock \href {https://doi.org/10.1198/016214504000001880} {Causal inference using potential outcomes}.
\newblock \emph{Journal of the American Statistical Association}, 100(469):322--331.

\bibitem[{Rust et~al.(2021)Rust, Pfeiffer, Vuli{\'c}, Ruder, and Gurevych}]{rust-etal-2021-good}
Phillip Rust, Jonas Pfeiffer, Ivan Vuli{\'c}, Sebastian Ruder, and Iryna Gurevych. 2021.
\newblock \href {https://doi.org/10.18653/v1/2021.acl-long.243} {How good is your tokenizer? {O}n the monolingual performance of multilingual language models}.
\newblock In \emph{Proceedings of the 59th Annual Meeting of the Association for Computational Linguistics and the 11th International Joint Conference on Natural Language Processing (Volume 1: Long Papers)}, pages 3118--3135, Online. Association for Computational Linguistics.

\bibitem[{Sch{\"a}fer et~al.(2024)Sch{\"a}fer, Hofmann, Schlag, and Pimentel}]{schafer-etal-2024-effect}
Anton Sch{\"a}fer, Thomas Hofmann, Imanol Schlag, and Tiago Pimentel. 2024.
\newblock \href {https://doi.org/10.18653/v1/2024.findings-acl.571} {On the effect of (near) duplicate subwords in language modelling}.
\newblock In \emph{Findings of the Association for Computational Linguistics: ACL 2024}, pages 9580--9597, Bangkok, Thailand. Association for Computational Linguistics.

\bibitem[{Schmidt et~al.(2024)Schmidt, Reddy, Zhang, Alameddine, Uzan, Pinter, and Tanner}]{schmidt-etal-2024-tokenization}
Craig~W. Schmidt, Varshini Reddy, Haoran Zhang, Alec Alameddine, Omri Uzan, Yuval Pinter, and Chris Tanner. 2024.
\newblock \href {https://doi.org/10.18653/v1/2024.emnlp-main.40} {Tokenization is more than compression}.
\newblock In \emph{Proceedings of the 2024 Conference on Empirical Methods in Natural Language Processing}, pages 678--702, Miami, Florida, USA. Association for Computational Linguistics.

\bibitem[{Schuster and Nakajima(2012)}]{schuster-nakajima-2012-voice}
Mike Schuster and Kaisuke Nakajima. 2012.
\newblock \href {https://doi.org/10.1109/ICASSP.2012.6289079} {{J}apanese and {K}orean voice search}.
\newblock In \emph{2012 IEEE International Conference on Acoustics, Speech and Signal Processing (ICASSP)}, pages 5149--5152.

\bibitem[{Sennrich et~al.(2016)Sennrich, Haddow, and Birch}]{sennrich-etal-2016-neural}
Rico Sennrich, Barry Haddow, and Alexandra Birch. 2016.
\newblock \href {https://doi.org/10.18653/v1/P16-1162} {Neural machine translation of rare words with subword units}.
\newblock In \emph{Proceedings of the 54th Annual Meeting of the Association for Computational Linguistics (Volume 1: Long Papers)}, pages 1715--1725, Berlin, Germany. Association for Computational Linguistics.

\bibitem[{Stahlberg and Byrne(2019)}]{stahlberg-byrne-2019-nmt}
Felix Stahlberg and Bill Byrne. 2019.
\newblock \href {https://doi.org/10.18653/v1/D19-1331} {On {NMT} search errors and model errors: Cat got your tongue?}
\newblock In \emph{Proceedings of the 2019 Conference on Empirical Methods in Natural Language Processing and the 9th International Joint Conference on Natural Language Processing (EMNLP-IJCNLP)}, pages 3356--3362, Hong Kong, China. Association for Computational Linguistics.

\bibitem[{Thistlewaite and Campbell(1960)}]{thistlewaite1960regression}
Donald~L. Thistlewaite and Donald~T. Campbell. 1960.
\newblock \href {https://doi.org/10.1037/h0044319} {Regression-discontinuity analysis: {A}n alternative to the ex-post facto experiment}.
\newblock \emph{Journal of Educational Psychology}, 51(6):309--317.

\bibitem[{Toraman et~al.(2023)Toraman, Yilmaz, \c{S}ahi\.nu\c{c}, and Ozcelik}]{Toraman_2023}
Cagri Toraman, Eyup~Halit Yilmaz, Furkan \c{S}ahi\.nu\c{c}, and Oguzhan Ozcelik. 2023.
\newblock \href {https://doi.org/10.1145/3578707} {Impact of tokenization on language models: {A}n analysis for {T}urkish}.
\newblock \emph{ACM Trans. Asian Low-Resour. Lang. Inf. Process.}, 22(4).

\bibitem[{Touvron et~al.(2023)Touvron, Martin, Stone, Albert, Almahairi, Babaei, Bashlykov, Batra, Bhargava, Bhosale, Bikel, Blecher, Ferrer, Chen, Cucurull, Esiobu, Fernandes, Fu, Fu, Fuller, Gao, Goswami, Goyal, Hartshorn, Hosseini, Hou, Inan, Kardas, Kerkez, Khabsa, Kloumann, Korenev, Koura, Lachaux, Lavril, Lee, Liskovich, Lu, Mao, Martinet, Mihaylov, Mishra, Molybog, Nie, Poulton, Reizenstein, Rungta, Saladi, Schelten, Silva, Smith, Subramanian, Tan, Tang, Taylor, Williams, Kuan, Xu, Yan, Zarov, Zhang, Fan, Kambadur, Narang, Rodriguez, Stojnic, Edunov, and Scialom}]{touvron2023llama}
Hugo Touvron, Louis Martin, Kevin Stone, Peter Albert, Amjad Almahairi, Yasmine Babaei, Nikolay Bashlykov, Soumya Batra, Prajjwal Bhargava, Shruti Bhosale, Dan Bikel, Lukas Blecher, Cristian~Canton Ferrer, Moya Chen, Guillem Cucurull, David Esiobu, Jude Fernandes, Jeremy Fu, Wenyin Fu, Brian Fuller, Cynthia Gao, Vedanuj Goswami, Naman Goyal, Anthony Hartshorn, Saghar Hosseini, Rui Hou, Hakan Inan, Marcin Kardas, Viktor Kerkez, Madian Khabsa, Isabel Kloumann, Artem Korenev, Punit~Singh Koura, Marie-Anne Lachaux, Thibaut Lavril, Jenya Lee, Diana Liskovich, Yinghai Lu, Yuning Mao, Xavier Martinet, Todor Mihaylov, Pushkar Mishra, Igor Molybog, Yixin Nie, Andrew Poulton, Jeremy Reizenstein, Rashi Rungta, Kalyan Saladi, Alan Schelten, Ruan Silva, Eric~Michael Smith, Ranjan Subramanian, Xiaoqing~Ellen Tan, Binh Tang, Ross Taylor, Adina Williams, Jian~Xiang Kuan, Puxin Xu, Zheng Yan, Iliyan Zarov, Yuchen Zhang, Angela Fan, Melanie Kambadur, Sharan Narang, Aurelien Rodriguez, Robert Stojnic, Sergey Edunov, and Thomas
  Scialom. 2023.
\newblock \href {https://arxiv.org/abs/2307.09288} {Llama 2: {O}pen foundation and fine-tuned chat models}.
\newblock \emph{arXiv preprint 2307.09288}.

\bibitem[{Vieira et~al.(2024)Vieira, LeBrun, Giulianelli, Gastaldi, DuSell, Terilla, O'Donnell, and Cotterell}]{vieira-etal-2024-language}
Tim Vieira, Ben LeBrun, Mario Giulianelli, Juan~Luis Gastaldi, Brian DuSell, John Terilla, Timothy~J. O'Donnell, and Ryan Cotterell. 2024.
\newblock \href {https://arxiv.org/abs/2412.03719} {From language models over tokens to language models over characters}.
\newblock \emph{arXiv preprint 2412.03719}.

\bibitem[{Wang et~al.(2025)Wang, Li, Jiang, Ding, Luo, Jiang, Liang, and Yang}]{wang2025tokenizationmattersdegradinglarge}
Dixuan Wang, Yanda Li, Junyuan Jiang, Zepeng Ding, Ziqin Luo, Guochao Jiang, Jiaqing Liang, and Deqing Yang. 2025.
\newblock \href {https://arxiv.org/abs/2405.17067} {Tokenization matters! {D}egrading large language models through challenging their tokenization}.
\newblock \emph{arXiv preprint 2405.17067}.

\bibitem[{Wen et~al.(2025)Wen, Li, Wang, Hall, Liang, and Ma}]{wen2024understanding}
Kaiyue Wen, Zhiyuan Li, Jason~S. Wang, David Leo~Wright Hall, Percy Liang, and Tengyu Ma. 2025.
\newblock \href {https://openreview.net/forum?id=m51BgoqvbP} {Understanding warmup-stable-decay learning rates: {A} river valley loss landscape view}.
\newblock In \emph{The Thirteenth International Conference on Learning Representations}.

\bibitem[{Whittington et~al.(2025)Whittington, Bachmann, and Pimentel}]{whittington2024tokenisation}
Philip Whittington, Gregor Bachmann, and Tiago Pimentel. 2025.
\newblock \href {https://arxiv.org/abs/2412.15210} {Tokenisation is {NP}-complete}.
\newblock In \emph{Proceedings of the 63rd Annual Meeting of the Association for Computational Linguistics (Volume 1: Long Papers)}. Association for Computational Linguistics.

\bibitem[{Wolf et~al.(2020)Wolf, Debut, Sanh, Chaumond, Delangue, Moi, Cistac, Rault, Louf, Funtowicz, Davison, Shleifer, {von Platen}, Ma, Jernite, Plu, Xu, Le~Scao, Gugger, Drame, Lhoest, and Rush}]{wolf-etal-2020-transformers}
Thomas Wolf, Lysandre Debut, Victor Sanh, Julien Chaumond, Clement Delangue, Anthony Moi, Pierric Cistac, Tim Rault, Remi Louf, Morgan Funtowicz, Joe Davison, Sam Shleifer, Patrick {von Platen}, Clara Ma, Yacine Jernite, Julien Plu, Canwen Xu, Teven Le~Scao, Sylvain Gugger, Mariama Drame, Quentin Lhoest, and Alexander Rush. 2020.
\newblock \href {https://doi.org/10.18653/v1/2020.emnlp-demos.6} {Transformers: {S}tate-of-the-art natural language processing}.
\newblock In \emph{Proceedings of the 2020 {{Conference}} on {{Empirical Methods}} in {{Natural Language Processing}}: {{System Demonstrations}}}, pages 38--45, {Online}. {Association for Computational Linguistics}.

\bibitem[{Zhai et~al.(2022)Zhai, Kolesnikov, Houlsby, and Beyer}]{zhai2022scaling}
Xiaohua Zhai, Alexander Kolesnikov, Neil Houlsby, and Lucas Beyer. 2022.
\newblock \href {https://openaccess.thecvf.com/content/CVPR2022/html/Zhai\%5FScaling\%5FVision\%5FTransformers\%5FCVPR\%5F2022\%5Fpaper.html} {Scaling vision transformers}.
\newblock In \emph{Proceedings of the IEEE/CVF conference on computer vision and pattern recognition}, pages 12104--12113.

\bibitem[{Zouhar et~al.(2023)Zouhar, Meister, Gastaldi, Du, Sachan, and Cotterell}]{zouhar-etal-2023-tokenization}
Vil{\'e}m Zouhar, Clara Meister, Juan Gastaldi, Li~Du, Mrinmaya Sachan, and Ryan Cotterell. 2023.
\newblock \href {https://doi.org/10.18653/v1/2023.acl-long.284} {Tokenization and the noiseless channel}.
\newblock In \emph{Proceedings of the 61st Annual Meeting of the Association for Computational Linguistics (Volume 1: Long Papers)}, pages 5184--5207, Toronto, Canada. Association for Computational Linguistics.

\end{thebibliography}
\clearpage

\appendix

\section{Definitions of Tokenisation Functions}
\label{app:tok_functions}

\paragraph{Merge-based Tokenisation Function.} 
The most common tokenisation function to date---used by, e.g., byte pair encoding (\bpe; \citealp{Gage1994,sennrich-etal-2016-neural})---builds subword-strings by merging a string's symbols two-at-a-time in a fixed pre-determined order.
The building blocks of this function are merges. 
A \textmergecolor{\defn{merge}} $\merge \in \vocab \times \vocab$ represents a pair of subwords; it can thus be written as: $\merge = \mergestring{\mergepair[1]}{\mergepair[2]}$. 
Given a list of merges $\merges \!=\! [\merge_1,\! \merge_2, \smalldots, \merge_{\size{\merges}}]$, this function is defined as:
\begin{align} \label{eq:bup_tok_function}
    \tokbup(\chs) \defeq 
    \mergefunc{\merge_{\size{\merges}}}\!(\cdots (\mergefunc{\merge_{1}}\!(\chs)))
\end{align}
where $\mergefunc{\merge}\colon \vocab^* \to \vocab^*$ is a function which scans a string left-to-right, replacing any appearance of $\mergepair[1]$ followed by $\mergepair[2]$ with another subword $\mergepair[\merge]$ which represents their concatenation $\mergepair[\merge] = \mergepair[1]\circ \mergepair[2]$.
E.g., $\mergefunc{\mergestring{he\!}{\!llo}}(\subwordstring{he,llo}) = \subwordstring{hello}$. 

\paragraph{Longest Prefix-match Tokenisation Function.}
A common alternative to \cref{eq:bup_tok_function}---used by, e.g., WordPiece (\wordpiece; \citealp{schuster-nakajima-2012-voice})---is 
to select the longest subword $\subword\in\vocab$ matching a prefix of $\chs$, and then recursively tokenising the remaining string.
This function is defined as:%
\begin{align}\label{eq:tok_func_long}
    \toklongest(\chs) \defeq {\color{\subwordcolor}\Big\langle} & \substack{\argmax_{\subword \in \vocab} \size{\subword} \\ \mathrm{s.t.}\, \chs[1:\size{\subword}] = \subword }, \toklongest(\chs[\size{\subword}:\size{\chs}]) {\color{\subwordcolor}\Big\rangle} 
\end{align}

\section{Proofs}

We provide the proofs for \cref{thm:no_effect,thm:initialisation_effect} here.

\subsection{Proof of \texorpdfstring{\Cref{thm:no_effect}}{Theorem 1}}
\label{app:no_effect_proof}

\noeffecttheorem*
\begin{proof}
Note that tokenisation functions $\tok$ are necessarily injective and, thus, each character-string is mapped to a unique subword-string.
For any tokeniser, we can thus rewrite \cref{eq:subword_ch_connection} as:
\begin{align}\label{eq:thm_ass}
    p(\chs) = \ptok(\tok(\chs))
\end{align}
Further, for perfect models $\pthetatok(\subwords) \mathop{=} \ptok(\subwords)$. 
We can thus show that:
\begin{subequations}
\begin{align}
    \pthetatok(\chs) 
    &= \sum_{\subwords \in \vocab^*} \pthetatok(\subwords) \; \one\{\chs = \detok(\subwords)\} \\
    &= \sum_{\subwords \in \vocab^*} \ptok(\subwords) \; \one\{\chs = \detok(\subwords)\} \\
    &= \ptok(\tok(\chs)) \\
    &= p(\chs)
\end{align}
\end{subequations} 
Under this condition, we can rewrite the potential outcome in \cref{eq:potential_outcome} as:
\begin{subequations}
\begin{align}
    \expectpotentialoutcome[\treatment] 
    &= \expect_{\chs[<t]}\left[
    \log \pthetatok(\chssubword \mid \chs[<t])
    \right] \\
    &= \expect_{\chs[<t]}\left[
        \log \frac{\pthetatok(\chs[<t] \circ \chssubword \circ \alphabet^*)}{\pthetatok(\chs[<t] \circ \alphabet^*)}   
    \right]\\
   &= \expect_{\chs[<t]}\left[
        \log \frac{p(\chs[<t] \circ \chssubword \circ \alphabet^*)}{p(\chs[<t] \circ \alphabet^*)}  
    \right]\\
   &= \expect_{\chs[<t]}\left[
        \log p(\chssubword \mid \chs[<t]) 
    \right] 
\end{align}
\end{subequations}
where we denote as $p(\chs \circ \alphabet^*)$ the sum of the probability assigned to all character-strings starting with $\chs$.
As the potential outcomes are a function of $p$, which does not depend on the used tokeniser, the values of $\expectpotentialoutcome[0]$ and $\expectpotentialoutcome[1]$ will be the same, and the causal effect $\atemerge$ will thus be 0.
\end{proof}

\subsection{Proof of \texorpdfstring{\Cref{thm:initialisation_effect}}{Theorem 2}}
\label{app:initialisation_effect}

{\renewcommand\footnote[1]{}\initialisationeffecttheorem*}
\begin{proof}
For treatment assignment 1, we have $\subword \in \tokeniser$ and $\tok(\chssubword) = \subword$;
for any context, thus, the model will approximately assign this character-string probability $\pthetatok(\subword \mid \subwords_{<t}) \mathop{=} \frac{1}{|\vocab|}$.\footnote{The character-string probabilities will only be approximately $\frac{1}{|\vocab|}$ since other subword-strings might also map to $\chssubword$. However, these alternative subword-strings will be longer, and thus have exponentially less probability mass. Once subword-string probabilities are marginalised out, they will thus make little difference to results.}
For treatment assignment 0, we have $\subword \notin \tokeniser$ and $\tok(\chssubword) = \subwordstring{\mergepair[1], \mergepair[2], ...}$, i.e.,  this subword now gets tokenised into multiple subwords instead;
for any context, thus, the model will approximately assign $\chssubword$ probability
\begin{equation}
\prod_{i=1}^{|\tok(\chssubword)|}\pthetatok(\mergepair[i] \!\mid \!\subwords_{<t} \!\circ\! \tok(\chssubword)_{<i}) \mathop{=} \frac{1}{|\vocab|}^{|\tok(\chssubword)|}
\end{equation}
For treatment assignment 0, we have $|\tok(\chssubword)| \geq 2$. The $\late$ is thus approximately lower-bounded by: $\log \frac{1}{|\vocab|} \!-\! \log \frac{1}{|\vocab|}^2 \!=\! \log |\vocab|$.
\end{proof}

\section{Implementation Details}
\label{app:implementation_details}

We implement all experiments using PyTorch \citep{paszke-etal-2019-pytorch} and implement variants of the Llama architecture using components implemented in the \myemph{transformers} library \citep{wolf-etal-2020-transformers}.

\paragraph{Data.}
We use the MiniPile dataset \citep{kaddour2023minipile}, a curated \q{6}{\gb} subset of the deduplicated Pile corpus \citep{gao-etal-2020-pile, biderman-etal-2022-datasheet}. The training set consists of \q{1}{\million} documents, which we use to train both tokenisers and models. Additionally, we train two \q{100}{\million}-parameter models on \q{20}{\billion} tokens from \fineweb \citep{penedo2024the}. 
For evaluation, we use the \q{10}{\thousand}-document validation set to collect the subwords' in-context log-probabilities $\pthetatok(\chssubword \mid \chs_{<t})$.

\paragraph{Tokenisers.}
Our method requires knowing which subwords would have been added if we allowed for a larger vocabulary.
To identify which subwords would be added with a larger vocabulary, we construct a tokeniser with a vocabulary size of \q{320}{\thousand}, then truncate it to define smaller tokenisers while retaining the full ranked list of subwords. 
Tokenisers are built using the \myemph{tokenisers} library,\footnote{\href{https://github.com/huggingface/tokenisers}{\myemph{github.com/huggingface/tokenisers}}.} and we encourage others to report similar details to support reproducibility.
We evaluate vocabularies of sizes \integer{8024}, \q{32}{\thousand}, and \q{128}{\thousand}, which allows us to study how tokenisation bias varies with vocabulary size. As the vocabulary grows, the added subwords are naturally of lower frequency.
We compare \bpe and \wordpiece tokenisers, trained identically using byte-level pre-tokenisers, processors, decoders, and a byte-based alphabet.

\paragraph{Model Training.}
We use variants of the LLaMA 2 architecture. Our default model has \q{57}{\million} parameters, with \integer{6} layers, \integer{24} attention heads, a hidden size of \integer{768}, and tied input–output embeddings. Details for other model configurations are available in our repository.\footnote{\href{https://github.com/pietrolesci/tokenisation-bias}{\myemph{github.com/pietrolesci/tokenisation-bias}}.}
Models are trained with AdamW \citep{adam,loshchilov2018decoupled} using learning rate \snum{6e-4}, $\beta_1\mathop{=}0.9$, $\beta_2\mathop{=}0.95$, $\epsilon\mathop{=}\snum{1e-8}$, and weight decay \float{0.1}. We adopt the warmup-stable-decay schedule \citep{zhai2022scaling}, which maintains a flat learning rate after warmup and applies cosine decay only during cooldown---an approach shown to be effective for small LMs \citep{hu2024minicpm, wen2024understanding} and which avoids requiring a fixed compute budget.
Training uses a context size of \integer{2048} tokens, batch size \integer{128}, gradient clipping to \float{1}, and runs for \q{50}{\thousand} steps with checkpoints saved every \q{2}{\thousand} steps. 
All experiments are run with the same seed for consistent data order and initialisation.

\paragraph{Regression Model.}
Given observed outcomes and subword indices, we estimate $\rdestimand$ by fitting the regression:
\begin{align}\label{eq:regression_model}
    \observedoutcome = \alpha + \beta \frac{\ordermerge}{1000} + \rdestimand \treatment_{\subword} + \residual_{\subword}
\end{align}
where $\residual_{\subword}$ is a zero-mean error term. To avoid confounding, we exclude subwords that are nested within larger subwords also present in the vocabulary (e.g., if both \subwordstring{he} and \subwordstring{hello} are included, only the latter is used).\footnote{Following standard causal analysis, we assume SUTVA: treating one unit does not affect others. This assumption is violated by overlapping subwords---i.e., subwords that are themselves part of larger subwords---which we thus exclude.}

\paragraph{Hardware Details.}
We use a server with one \myemph{NVIDIA A100 80GB PCIe}, \integer{32} CPUs, and \integer{32} GB of RAM for all experiments. Below, we report a subset of the output of the \myemph{lscpu} command:

\begin{tcolorbox}[left=5pt,right=5pt,top=5pt,bottom=5pt]
    \small
    \begin{verbatim}
Architecture:        x86_64
CPU op-mode(s):      32-bit, 64-bit
Address sizes:       46 bits physical, 
                     48 bits virtual
Byte Order:          Little Endian
CPU(s):              32
On-line CPU(s) list: 0-31
Vendor ID:           GenuineIntel
Model name:          Intel(R) Xeon(R)
                     Silver 4210R CPU
                     @ 2.40GHz
CPU family:          6
Model:               85
Thread(s) per core:  1
Core(s) per socket:  1
Socket(s):           8
Stepping:            7
BogoMIPS:            4800.11
\end{verbatim}
\end{tcolorbox}

\paragraph{Reproducibility.}
We release all experimental artefacts as a HuggingFace Hub collection.\footnote{\href{https://huggingface.co/collections/pietrolesci/tokenisation-bias-66d5d0b40cb82a2d789b19db}{\myemph{https://huggingface.co/collections/pietrolesci/tokenisation-bias-66d5d0b40cb82a2d789b19db}}.} 
This includes: (i) the raw and tokenised data in model-consumed order; (ii) the full \q{320}{\thousand}-subword tokenisers; (iii) the reduced-vocabulary tokenisers used in our experiments; and (iv) all model checkpoints. We encourage others to adopt similar practices and make tokeniser design choices more transparent.

\section{Additional Plots}
\label{app:addtional_plots}

Additional plots follow on the next pages.

\begin{figure*}[!ht]
    \centering
    \includegraphics[width=\linewidth]{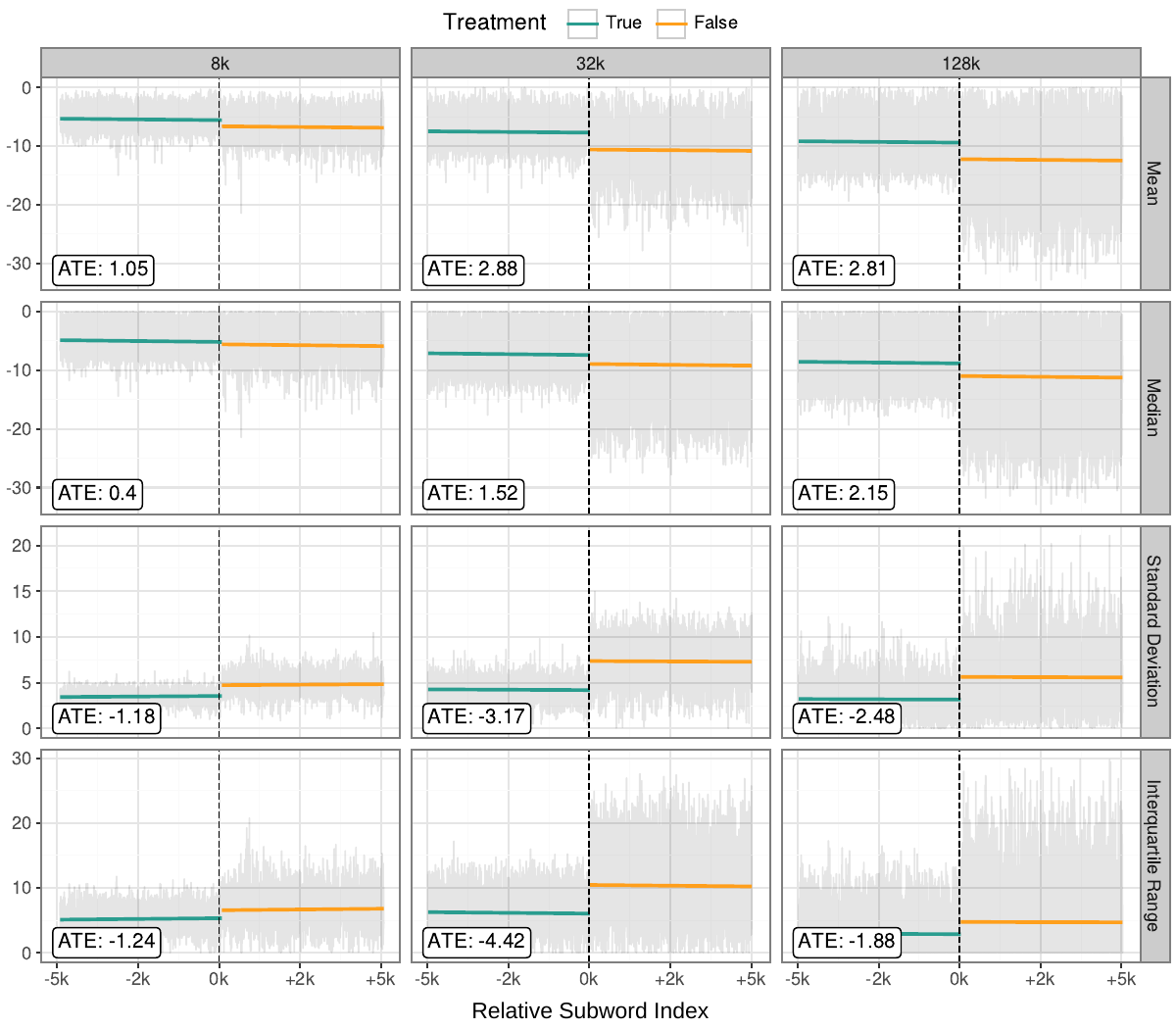}\vspace{-5pt}
    \caption{Average treatment effect for \bpe with $\cutoff \in \{\q{8}{\thousand}, \q{32}{\thousand}, \q{128}{\thousand}\}$ at the last model checkpoints. Each row refers to a different outcome variable: mean, standard deviation, median, and interquartile range of a $\chssubword$’s log-probability across contexts. Subwords on the left-hand side of the cutoff are treated (i.e., added to the vocabulary).\looseness=-1}%
    \label{fig:bpe_vocabsize_metrics}
    \vspace{-8pt}
\end{figure*}

\begin{figure*}[!ht]
    \centering
    \includegraphics[width=\linewidth]{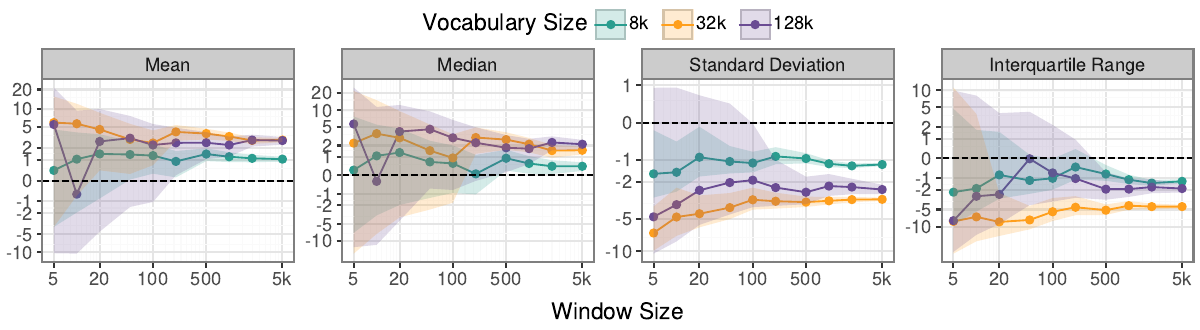}\vspace{-5pt}
    \caption{
    Estimated effect $\rdestimator$ (y-axis) vs.\ window size used to estimate it. Results for fully trained model with \bpe and $\cutoff\in\{\q{8}{\thousand}, \q{32}{\thousand}, \q{128}{\thousand}\}$. Shaded regions correspond to standard errors.
    The columns refer to different outcome variables: mean, standard deviation, median, and interquartile range of a $\chssubword$’s log-probability across contexts.\looseness=-1}
    \label{fig:bpe_windowsizes}
    \vspace{-8pt}
\end{figure*}

\begin{figure*}[!ht]
    \centering\small
    \includegraphics[width=\linewidth]{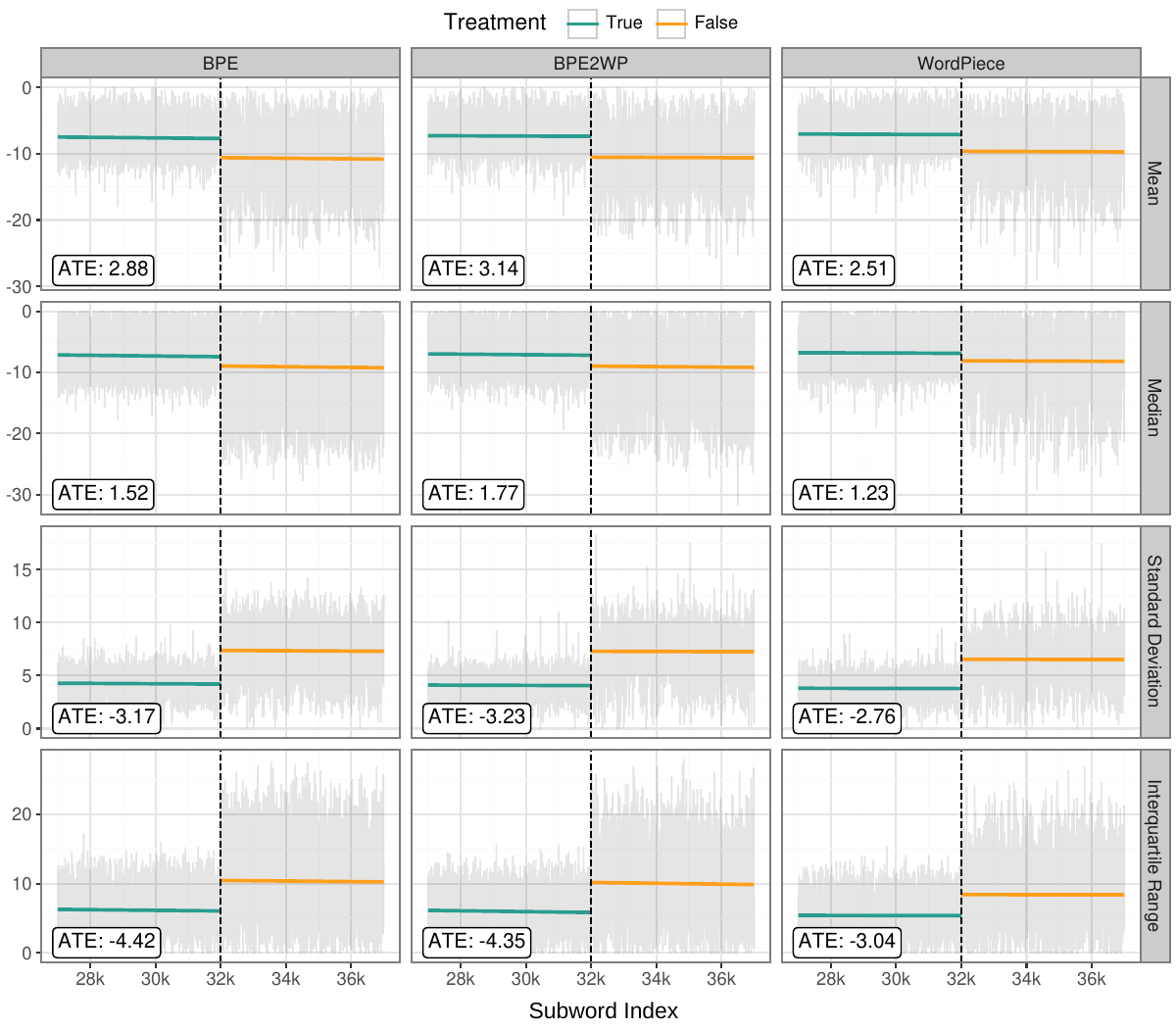}
    \vspace{-15pt}
    \caption{Average treatment effect for fully trained models with \bpe, \bpewp, and \wordpiece and $\cutoff=\q{32}{\thousand}$. Each row refers to a different outcome variable: mean, standard deviation, median, and interquartile range of a $\chssubword$’s log-probability across contexts. Subwords on the left-hand side of the cutoff are treated (i.e., added to the vocabulary).\looseness=-1}
    \label{fig:alltokenisers32000_metrics}
\end{figure*}

\begin{figure*}[!ht]
    \centering
    \includegraphics[width=\linewidth]{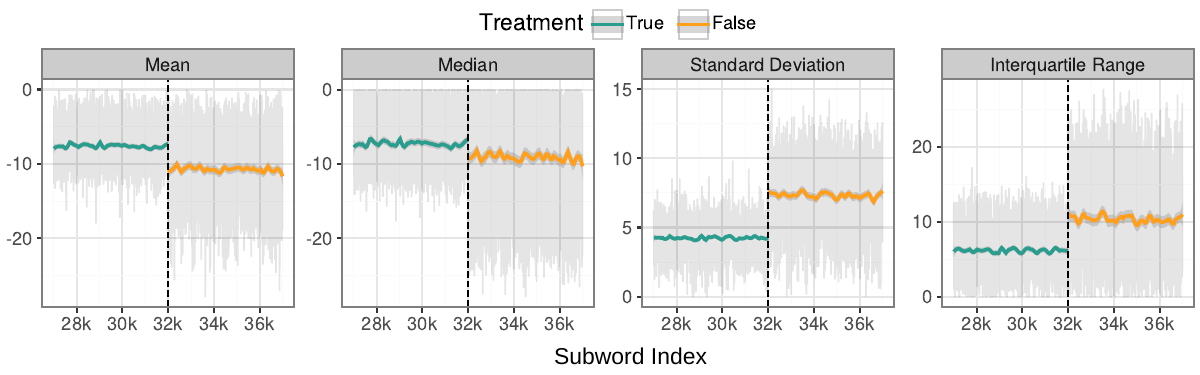}
    \caption{Stability of the average treatment effect with respect to the functional form of $\funcrel$, for fully trained models with \bpe and $\cutoff=\q{32}{\thousand}$. Columns refer to different outcome variables: mean, standard deviation, median, and interquartile range of a $\chssubword$’s log-probability across contexts. 
    Subwords on the left-hand side of the cutoff are treated (i.e., added to the vocabulary). 
    Conditional mean lines and confidence intervals are computed using the \myemph{LOESS} (locally estimated scatterplot smoothing) method.}
    \label{fig:loess_regression_model}
\end{figure*}

\end{document}